\documentclass[sigconf]{acmart}
\newtheorem{definition}{Definition} 
\usepackage[table]{xcolor} 
\usepackage{pifont} 
\usepackage{subcaption}
\usepackage{algorithm}
\usepackage{algorithmic}
\usepackage{xspace}
\usepackage{enumitem}
\usepackage{graphics}
\usepackage{mathtools}
\usepackage{url}
\AtBeginDocument{%
  }
\allowdisplaybreaks 

\copyrightyear{2026}
\acmYear{2026}
\setcopyright{cc}
\setcctype{by}
\acmConference[WWW '26]{Proceedings of the ACM Web Conference 2026}{April 13--17, 2026}{Dubai, United Arab Emirates}
\acmBooktitle{Proceedings of the ACM Web Conference 2026 (WWW '26), April 13--17, 2026, Dubai, United Arab Emirates}
\acmPrice{}
\acmDOI{10.1145/3774904.3792175}
\acmISBN{979-8-4007-2307-0/2026/04}

\begin{document}
\setlength{\textfloatsep}{1.0pt}
\newcommand{\unclear}[1]{{\color{red}{#1}}\noindent}

\renewcommand{\dblfloatpagefraction}{1.0}
\newcommand{\TSK}[1]{#1}
\newcommand{\tensakuClean}{
    \renewcommand{\TSK}[1]{}
}

\newcommand{\vswd}{\vspace{0.0em}}
\newcommand{\bit}{\begin{itemize}}
\newcommand{\eit}{\end{itemize}}
\newcommand{\ben}{\vspace{-0em}\begin{enumerate}}
\newcommand{\een}{\end{enumerate}}
\newcommand{\bal}{\begin{align}}
\newcommand{\eal}{\end{align}}
\newcommand{\beq}{\vspace{-0.0em}\begin{equation}}
\newcommand{\eeq}{\end{equation}}

\newcommand{\argmax}{\mathop{\rm arg~max}\limits}
\newcommand{\argmin}{\mathop{\rm arg~min}\limits}
\newcommand{\emp}{\bf \underline}
\newcommand{\mrk}{$\surd$}  
\newcommand{\R}{\mathbb{R}}  
\newcommand{\C}{\mathbb{C}}  
\newcommand{\N}{\mathbb{N}}  

\newcommand{\goal}[1]{ {\noindent {$\Rightarrow$} \em {#1} } }
\newcommand{\hide}[1]{}
\newcommand{\reminder}[1]{\textsf{\textcolor{red}{#1}}}
\newcommand{\reminderrm}[1]{\textsf{\textcolor{red}{[#1]}}}
\newcommand{\add}[1]{\textsf{\textcolor{blue}{#1}}}
\newcommand{\duty}[1]{{\textsf{\textcolor{blue}{[#1's job]}}}}
\newcommand{\vectornorm}[1]{\left|\left|#1\right|\right|}
\newcommand{\mypara}[1]{\vspace{1.00em}\noindent\textbf{#1.}}
\newcommand{\myparabf}[1]{\vspace{0.00em}\noindent\textbf{#1}}
\newcommand{\myparaitemize}[1]{\noindent{\textbf{#1.}}}

\newcommand{\bi}{\bfseries\itshape}
\newcommand{\ac}[1]{\acute{#1}}

\newcommand{\prop}[1]{({\bf{#1}})\xspace}

\newtheorem{informalProblem}{InformalProblem}
\newtheorem{observation}{\textsc{Observation}}
\newtheorem{problem}{\textsc{Problem}}



\newcommand{\tensor}{\mathcal{X}}
\newcommand{\tensorE}{\mathcal{X}^{E}}
\newcommand{\tensorC}{\mathcal{X}^{C}}
\newcommand{\tensorF}{\mathcal{X}^{F}}
\newcommand{\INFTY}{\mathbb {R}}
\newcommand{\tensorShape}{\prod_{\ldmode=1}^{\ndmode}\nunits_{\ldmode} \times\prod_{\lcmode=1}^{\ncmode}\INFTY}

\newcommand{\Matt}{\mathbf{A}}
\newcommand{\Mtime}{\mathbf{B}}
\newcommand{\Mcatt}{\mathbf{C}}

\newcommand{\Patt}{\mathbf{\hat{A}}}
\newcommand{\Pcatt}{\mathbf{\hat{C}}}
\newcommand{\Ptime}{\mathbf{\hat{B}}}

\newcommand{\Vatt}{\bm{a}}
\newcommand{\Vtime}{\bm{b}}
\newcommand{\PVatt}{\bm{\hat{a}}}
\newcommand{\PVtime}{\bm{\hat{b}}}

\newcommand{\Eatt}{{a}}
\newcommand{\Etime}{{b}}
\newcommand{\patt}{\hat{a}}
\newcommand{\ptime}{\hat{b}}

\newcommand{\nunits}{U}
\newcommand{\nmode}{M}
\newcommand{\ngrid}{G}
\newcommand{\ndmode}{{M_1}}
\newcommand{\ncmode}{{M_2}}
\newcommand{\ntopic}{K}
\newcommand{\ntime}{T}
\newcommand{\nctime}{{T_c}}
\newcommand{\nevent}{{N_\ltime}}
\newcommand{\nepoch}{E}
\newcommand{\nlfbgs}{I}
\newcommand{\Ntime}{{N_\ltime}}
\newcommand{\Ntopic}[1][]{{\mathbf{N}_{#1}^{(\ntopic)}}}
\newcommand{\Ntopicmode}[1][]{\mathbf{N}_{#1}^{(\ldmode)}}
\newcommand{\Ntopicgrid}[1][]{\mathbf{N}_{#1}^{(\lcmode)}}
\newcommand{\Ntimetopic}{N_{\ltime,\ltopic}}
\newcommand{\Stopic}{\mathbf{S}^{(\ntopic)}}
\newcommand{\Stopicmode}{\mathbf{S}^{(\ldmode)}}
\newcommand{\Stopicgrid}{\mathbf{S}^{(\lcmode)}}
\newcommand{\Stimetopic}{\mathbf{S}_{\ltime,\ltopic}}

\newcommand{\lunit}{u}
\newcommand{\lmode}{m}
\newcommand{\lgrid}{g}
\newcommand{\ldmode}{{m_1}}
\newcommand{\lcmode}{{m_2}}
\newcommand{\ltopic}{k}
\newcommand{\ltime}{t}
\newcommand{\levent}{n}

\newcommand{\bevent}{\mathbf{e}}
\newcommand{\event}{e}
\newcommand{\assignment}{z_{\ltime, \levent}}
\newcommand{\observed}{y}
\newcommand{\Ktparam}{\Mtime}
\newcommand{\Kcparam}{\Mcatt}
\newcommand{\KernelFunc}{\kappa}
\newcommand{\Ktime}{\KernelFunc_{\Ktparam}}
\newcommand{\Kcatt}{\KernelFunc_{\Kcparam}}
\newcommand{\KernelMat}{\mathbf{K}}
\newcommand{\stamp}{\tau}
\newcommand{\startTime}{\ltime_s}

\newcommand{\dpersistency}{\alpha^{(\ldmode)}}

\newcommand{\Norder}{p}
\newcommand{\Gain}{\mathbf{R}}
\newcommand{\latmean}{\mathbf{m}}
\newcommand{\latcov}{\mathbf{P}}
\newcommand{\predmean}{\latmean^p}
\newcommand{\predcov}{\latcov^p}
\newcommand{\filtmean}{\latmean^f}
\newcommand{\filtcov}{\latcov^f}
\newcommand{\smoothmean}{\latmean^s}
\newcommand{\smoothcov}{\latcov^s}
\newcommand{\obsmean}{\mu}
\newcommand{\obsvar}{\sigma^2}
\newcommand{\invobsvar}{(\sigma^{2})^{-1}}
\newcommand{\polyagamma}{\omega}
\newcommand{\lineartrans}{\boldsymbol{\Phi}}
\newcommand{\Matobservation}{\mathbf{H}}
\newcommand{\pgmean}{(\obsmean_{\ltime,\ltopic})_{pg}}
\newcommand{\pgvar}{(\sigma^2_{\ltime,\ltopic})_{pg}}
\newcommand{\obsnoise}{\epsilon}
\newcommand{\noisevar}{\sigma_{noise}^2}

\newcommand{\mathN}{\mathbb{N}}
\newcommand{\shapeN}{\in\mathN}
\newcommand{\mathR}{\mathbb{R}}
\newcommand{\shapeR}{\in\mathR}

\newcommand{\method}{\textsc{HeteroComp}\xspace}
\newcommand{\streamName}{heterogeneous tensor stream\xspace}
\newcommand{\streamNames}{heterogeneous tensor streams\xspace}
\newcommand{\topic}{component\xspace}
\newcommand{\Topic}{Component\xspace}
\newcommand{\topics}{components\xspace}
\newcommand{\attribute}{attribute\xspace}
\newcommand{\Attribute}{Attribute\xspace}
\newcommand{\attributes}{attributes\xspace}
\newcommand{\Attributes}{Attributes\xspace}

\newcommand{\record}{record\xspace}
\newcommand{\records}{records\xspace}
\newcommand{\cmode}{continuous \attribute}
\newcommand{\Cmode}{Continuous \Attribute}
\newcommand{\cmodes}{continuous \attributes}
\newcommand{\Cmodes}{Continuous \Attributes}
\newcommand{\dmode}{categorical \attribute}
\newcommand{\Dmode}{Categorical \Attribute}
\newcommand{\dmodes}{categorical \attributes}
\newcommand{\Dmodes}{Categorical \Attributes}
\newcommand{\dcmode}{categorical and continuous \attribute}
\newcommand{\dcmodes}{categorical and continuous \attributes}
\newcommand{\anomscore}{\text{score}(\tensorC)}
\newcommand{\Parameters}{\Theta}
\newcommand{\Statistics}{\mathcal{S}}
\newcommand{\FullParameters}{\mathcal{F}}
\newcommand{\grid}{\Delta^{(\lcmode)}}
\newcommand{\gridx}{mean(\grid_g)}
\newcommand{\gridy}{c}
\newcommand{\region}{{V^{(\lcmode)}}}
\newcommand{\varc}{{\sigma^2_C}}
\newcommand{\interval}{\delta}
\newcommand{\sinterval}{\mathcal {T}}
\newcommand{\cinterval}{\delta_c}
\newcommand{\inovationMean}{r_\lgrid}
\newcommand{\inovationVar}{\sigma^2_{r_\lgrid}}


\newcommand{\ciseven}{\#1 \textit{CI'17}\xspace}
\newcommand{\pciseven}{(\#1) \textit{CI'17}\xspace}
\newcommand{\cieight}{\textit{\#2 CCI'18}\xspace}
\newcommand{\pcieight}{\textit{(\#2) CCI'18}\xspace}

\newcommand{\edge}{\textit{\#3 Edge-IIoT}\xspace}
\newcommand{\pedge}{\textit{(\#3) Edge-IIoT}\xspace}
\newcommand{\ddos}{\textit{\#4 DDos2019}\xspace}
\newcommand{\pddos}{\textit{(\#4) DDos2019}\xspace}
\newcommand{\cupid}{\textit{\#5 CUPID}\xspace}
\newcommand{\pcupid}{\textit{(\#5) CUPID}\xspace}
\newcommand{\amazon}{\textit{\#6 Amazon Movie\&TV}\xspace}
\newcommand{\pamazon}{\textit{(\#6) Amazon Movie\&TV}\xspace}

\newcommand{\roc}{\textit{AUC-ROC}\xspace}
\newcommand{\pr}{\textit{AUC-PR}\xspace}

\newcommand{\PaperTitle}{Multi-Aspect Mining and Anomaly Detection for\\ Heterogeneous Tensor Streams}
\title[{\let\\\relax \PaperTitle}]{\texorpdfstring{\PaperTitle}{}}

\author{Soshi Kakio}
\affiliation{%
  \institution{SANKEN, The University of Osaka}
  \city{Osaka}
  \country{Japan}
}
\email{skakio88@sanken.osaka-u.ac.jp}

\author{Yasuko Matsubara}
\affiliation{%
  \institution{SANKEN, The University of Osaka}
  \city{Osaka}
  \country{Japan}
}
\email{yasuko@sanken.osaka-u.ac.jp}

\author{Ren Fujiwara}
\affiliation{%
    \institution{SANKEN, The University of Osaka}
    \city{Osaka}
    \country{Japan}
}
\email{r-fujiwr88@sanken.osaka-u.ac.jp}

\author{Yasushi Sakurai}
\affiliation{%
  \institution{SANKEN,  The University of Osaka}
  \city{Osaka}
  \country{Japan}
}
\email{yasushi@sanken.osaka-u.ac.jp}

\renewcommand{\shortauthors}{Soshi Kakio,  Yasuko Matsubara, Ren Fujiwara, and Yasushi Sakurai}

\begin{abstract}
Analysis and anomaly detection in event tensor streams consisting of timestamps and multiple \attributes—such as communication logs (time, IP address, packet length)—are essential tasks in data mining.
While existing tensor decomposition and anomaly detection methods provide useful insights, they face the following two limitations. (i) They cannot handle \streamNames, which comprises both \dmodes(e.g., IP address) and \cmodes (e.g., packet length).
They typically require either discretizing \cmodes or treating \dmodes as continuous, both of which distort the  underlying statistical properties of the data. Furthermore, incorrect assumptions about the distribution family of \cmodes often degrade the model’s performance. (ii) They discretize timestamps, failing to track the temporal dynamics of streams (e.g., trends, abnormal events), which makes them ineffective for detecting anomalies at the group level, referred to as "group anomalies" (e.g, DoS attacks).
To address these challenges, we propose \method, a method for continuously summarizing \streamNames into "\topics" representing latent groups in each \attribute and their temporal dynamics, and detecting group anomalies. 
Our method employs Gaussian process priors to model unknown distributions of \cmodes, and  temporal dynamics, which directly estimate probability densities from data.
Extracted \topics give concise but effective summarization, enabling accurate group anomaly detection.
Extensive experiments on real datasets demonstrate that \method outperforms the state-of-the-art algorithms for group anomaly detection accuracy, and its computational time does not
depend on the data stream length.

\end{abstract}




\maketitle


\section{Introduction}
    \label{section:introduction}
    \TSK{
    
\begin{figure*}[ht]
    \centering 
    \includegraphics[width=\textwidth]{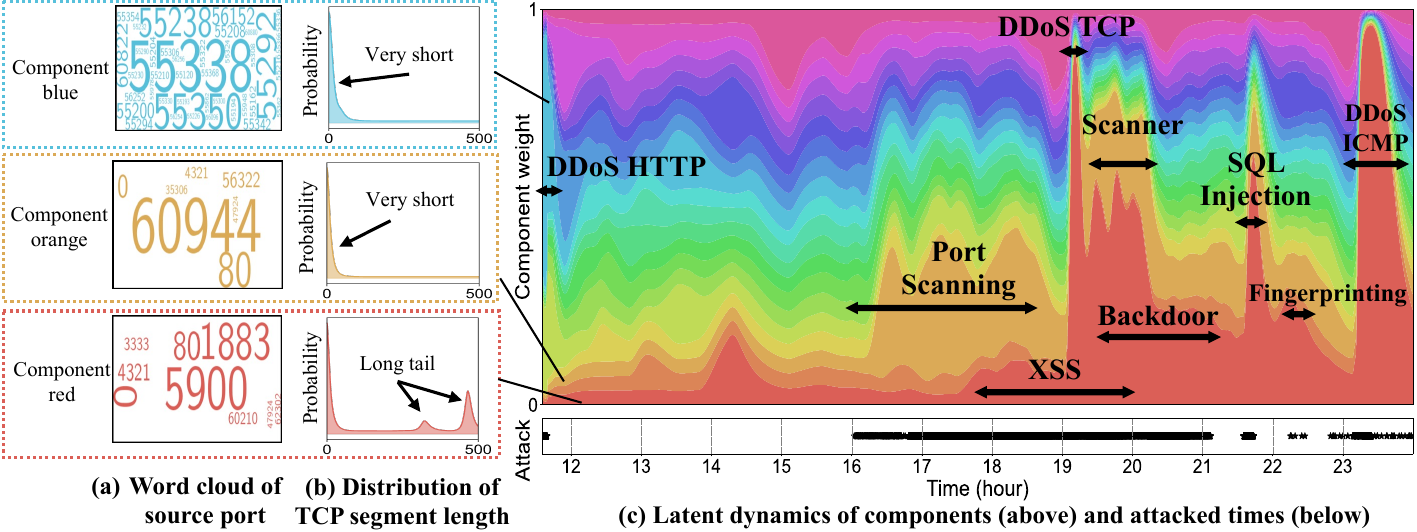}
    \caption{Modeling power of \method over \pedge dataset. Our proposed method can find the hidden \topics which represents different characteristics in both (a) \dmode (source port) and (b) \cmode (TCP segment length), and (c) \topic weight exhibit significant changes when cyber-attacks occurs.}
    \label{fig:preview}
    \Description{}
    \vspace{-10pt} 
\end{figure*}


}
The rapid development of information systems has made it possible to obtain a variety of multi-aspect event data streams, which consist of a timestamp and multiple \attributes (e.g., price, user ID, item name).
Crucially, the effective analysis and anomaly detection of such streams have numerous real-world applications \cite{CubeScope}, including online marketing analytics \cite{Trimine}, location-based services \cite{SMF,SSMF, DMPP}, and cybersecurity systems \cite{DenseAlert, MStream, MemStream, CyberCScope, Anograph}.
For example, marketers want to discover hidden groups of each \attribute and their trends from user's review data.
Furthermore, in cybersecurity systems, analyzing and detecting anomalies (e.g., DDoS attacks) from access logs as quickly as possible is crucial to minimize the damage from them.
Since there is no anomaly label in streaming settings, in this work, we focus on the analysis and unsupervised anomaly detection of multi-aspect data streams.
\par
These multi-aspect data streams are represented as high dimensional tensor streams, which are inherently sparse \cite{CubeScope},
where the number of present \records is much smaller than the tensor size.
Despite the wide success of existing tensor decomposition algorithms \cite{TensorSurvey}, which aim to reveal hidden structures and relationships of tensors,
this sparsity derails typical tensor decomposition methods because they are designed for dense tensors.
Recent studies focused on the streaming decomposition of sparse tensors \cite{OnlineSCP, SMF, SSMF, SliceNStich}, with an effective strategy being the \topics-based methods \cite{Trimine, CubeScope, CyberCScope} that capture hidden groups in each \attribute and their relationships.
However, most existing methods can handle only \dmodes (e.g., IP address, port, user ID), and cannot handle \cmodes  (e.g., price, flow duration, packet size).
Here, we refer to such tensor streams that consist of both \dmodes  and \cmodes as \textit{\streamName}.
To handle \streamNames, existing methods require discretizing \cmodes, which disrupts the continuity of \cmodes.
Furthermore, in real-world scenarios, we often have no information about the \cmodes; that is, we do not know the distributions of \cmodes.
Incorrect assumptions about the distribution family often degrade the model's performance.
\par
Typical streaming unsupervised anomaly detection methods \cite{IForestASD, RRCF, xStream} are effective at identifying anomalous \records (\textit{point anomalies}).
However, as they ignore the temporal relationships between \records, they cannot effectively detect \textit{group anomalies} (also referred to as collective anomalies), which may not be anomalies by themselves, but their occurrence together as a collection exhibits unusual patterns that deviate from the entire data set \cite{GraphSurvey}.
A canonical example of group anomalies is a DoS attack: an individual request could possibly be normal, but their high-frequency aggregation over a short duration constitutes a critical threat.
Although methods for detecting group anomalies \cite{Anograph, DenseAlert, AugSplicing} have been developed, they cannot handle \streamNames; they enforce data homogeneity by either discretizing \cmodes or treating \dmodes as continuous, both of which distort the underlying statistical properties of the data.
Furthermore, by discretizing timestamps, they fail to preserve their continuity.
In summary of the above discussion, we wish to solve the following problem:
\textit{Given a \streamName, how can we find hidden structures of the stream without restricting to any specific parameterized form of \cmodes, and quickly and accurately detect group anomalies?}
\par
In this paper, to tackle the above challenging problems, we propose \method
\footnote{Our source code is publicly available in \url{https://github.com/kaki005/HeteroComp}.},
for continuously summarizing \streamNames into \topics and their temporal dynamics, and detect group anomalies based on the \topics without restricting to any specific parameterized form.
Specifically, to model the unknown distribution of \cmodes, \method uses logistic Gaussian process priors \cite{LogGPPrior}, which directly estimate probability densities from data, thus \method can treat \cmodes uniformly.
In addition, \method models the \topics' latent dynamics using the Gaussian process, which utilizes the continuity of timestamps and captures complex temporal evolutions(changes) of \topics linked to external events.
Extracted \topics naturally represent latent groups in both \dcmodes, and relationship between \attributes, which provide an easy-to-understand summary of data.
Similar \records tend to cluster into the same \topic, which enables accurate  group anomaly detection by aggregating abnormal \records.
The framework further admits streaming updates without retraining from scratch, making it suitable for long-running deployments where \streamNames arrive continuously, and enabling low computational cost incorporation of new \records.

\subsection{Preview of Our Result}
Fig. \ref{fig:preview} shows an example of the analysis of \streamName (i.e., \pedge) using \method.
This dataset consists of \records with \dmodes, \cmodes, and timestamps.
Our method captures the following properties:
\begin{itemize}[leftmargin=*]
    \item \textbf{Modeling Heterogeneity}:
    Fig. \ref{fig:preview}(a)(b) shows the characteristics of three \topics blue, orange, and red.
    First, Fig. \ref{fig:preview}(a) shows the word clouds of source port \attribute.
    A larger size in the word cloud denote a stronger relationship with the \topic.
    \Topic blue is associated with \records sent from ports 55338 and 55350 whereas \topic orange is dominated by \records on port 60944, and \topic Red have \records sent from ports 5900 and 1883.
    Next, Fig. \ref{fig:preview}(b) shows the distribution of the length of the TCP segment \attribute.
    \Topic blue and orange encompass \records featuring short TCP segments, while \topic red primarily represents \records with long TCP segments.
    \item \textbf{Latent Dynamics}: Fig. \ref{fig:preview}(c) visualizes the latent dynamics of \topics, where the area of each color represents the  \topic assignment probability at each time.
    The dynamics of these latent \topics exhibit significant changes when cyber-attacks occurs.
    Specifically,  \topic blue becomes dominant over others when a DDoS HTTP attack occurs, while \topic orange becomes dominant during an attack originating from a different source port, such as Port Scanning and Vulnerability Scanner.
    Similarly, \topic red tends to increase during attacks characterized by the transmission of many packets with large TCP segment sizes, such as DDoS TCP, DDoS ICMP.
\end{itemize}

\myparaitemize{Contributions}
In this paper, we propose \method, which has the following desirable properties:
\begin{itemize}[leftmargin=*]
    \item \textit{Effective}: : Our proposed model summarizes \streamNames without restricting to any specific parameterized form, which extracts interpretable latent \topics and their latent dynamic (i.e., Fig. \ref{fig:preview}).
    \item \textit{Accurate}:  Extensive experiments on real-world datasets shows that \method outperforms baseline approaches for detecting group anomalies accurately.
    \item \textit{Scalable}:  Our proposed algorithm is fast and its computation time does not depend on the entire stream length.
\end{itemize}

\section{Related Work}
    \label{section:related_works}
    \TSK{
    \newcommand*\rot{\rotatebox{90}}
\newcommand*\OK{\ding{51}}
\newcommand*\NO{-}
\newcommand{\SOME}{some}

\begin{table}[t]
    \centering
    \caption{Capabilities of approaches.}
    \label{table: capability}
    \resizebox{1.0\linewidth}{!}{
        \begin{tabular}{l|cccc|ccc|c}
            \toprule
            &
            \rot{RRCF\cite{RRCF}}&
            \rot{MStream\cite{MStream}}&
            \rot{MemStream\cite{MemStream}}&
            \rot{Anograph\cite{Anograph}}&
            \rot{Trimine\cite{Trimine}} &
            \rot{CubeScope\cite{CubeScope}} &
            \rot{CyberCScope\cite{CyberCScope}} &
            \rot{\textbf{\method}} \\
            \midrule
            \rowcolor{lightgray}
            Anomaly detection  &\OK &\OK &\OK &\OK &\NO &\OK &\OK &\OK \\
            Multi-aspect mining &\NO &\NO &\NO &\NO &\OK &\OK &\OK &\OK \\
    
            \midrule
            \rowcolor{lightgray}
            Stream processing &\OK &\OK &\OK&\OK &\NO &\OK &\OK  &\OK \\
    
    
            Heterogeneous &\NO &\OK &\NO &\NO &\NO  &\NO &\OK &\OK \\
            
            \rowcolor{lightgray}
    
            Latent dynamics &\NO &\NO &\NO&\NO &\NO &\NO &\NO &\OK \\
    
            \bottomrule
        \end{tabular}
    }
\end{table}
}
In this section, we briefly describe investigations related to this research.
Table \ref{table: capability} shows the relative advantages of our method, and only \method meets all the requirements.

\myparaitemize{Sparse Tensor Decomposition}
A wide range of studies have been conducted on analyzing sparse tensors,
including probabilistic generative models\cite{BPTF, Trimine, Zhe2018-aw, Tillinghast2021-bc, Wang2020-lr, Wang2022-eh} and neural-based models\cite{NewuralAT}.
In particular, streaming algorithms have become more critical in terms of processing a substantial amount of data under time/memory limitations, and they have proved highly significant to the data mining and database community \cite{SliceNStich, OnlineSCP, POST, CPStream}.
CubeScope \cite{CubeScope} can summarize an event tensor stream interpretably, such as distinct patterns that change over time or major trends in \dmodes.
However, they can handle only \dmodes, so they enforce data homogeneity by discretizing \cmodes, which disrupts the continuity of \cmodes.
Only CyberCScope \cite{CyberCScope} can handle \streamNames by distinguishing between \dmodes and \cmodes, but it only supports the case where the \cmodes follow Gamma distributions, which may diminish \topic diversity.
Furthermore, it uses Multinomial distribution to model \topic evolution over time, failing to capture complex dynamics.

\myparaitemize{Streaming Anomaly Detection}
Over decades, popular anomaly detection methods have been extended to work online on a data stream\cite{NETS, MDUAL, ILOF, MILOF,DILOF},
including
One Class SVM \cite{OCSVM}, isolation forest \cite{IForestASD, RRCF}, and  deep learning methods \cite{ARCUS}.
However, they are designed for point anomaly detection, failing to detect group anomalies.
Although MemStream \cite{MemStream} can handle the time-varying data distribution known as concept drift \cite{conceptdriftsurveyIEEE2018} and robust to group anomalies, it cannot handle \dmodes.
Recent methods for group anomaly detection of multi-aspect data stream are mainly categorized into graph-based methods\cite{Midas, GraphSurvey, Anograph} which aim to detect node/edge/graph anomalies in graph streams, and tensor-based methods \cite{CrossSpot, DenseAlert, AugSplicing} which aim to detect suddenly appearing dense subtensors in sparse tensor streams.
However, they are designed for \dmodes, failing to handle \streamNames.
Although Mstream \cite{MStream} can detect group anomalies from multi-aspect data involving \dcmodes, it discretizes timestamps, which degrades detection accuracy.
Furthermore, they cannot summarize the streams interpretably.

\par
\myparaitemize{Deep Generative Models}
Although many deep generative models have been proposed to model heterogeneous tabular data\cite{HIVAE, CTGAN}
and non-parametric distributions \cite{NormalizingFlow},
they are computationally expensive for streaming scenarios.

\section{Proposed Method}
    \label{section:model}
\subsection{Problem Settings}
We continuously monitor a stream of \records $\{\bevent_1, \bevent_2,\dots\}$, arriving in a streaming manner.
Each \record $\bevent_i$ consists of timestamps $\stamp_i$, $\ndmode$ \dmodes $\{\event^{(\ldmode)}_i\}_{\ldmode=1}^{\ndmode}$, and $\ncmode$ \cmodes $ \{\event^{(\lcmode)}_i\}_{\lcmode=1}^{\ncmode}$ .
For the $\ldmode$-th \dmode, we assume a finite $\nunits_\ldmode$-th dimensional space, whereas for the $\lcmode$-th \cmode, we assume a real space $\mathbb R$.
This stream takes the form of a $(1+ \ndmode + \ncmode)$-th order tensor $\tensor\shapeN^{\ntime\times\tensorShape}$, where $\ntime$ is the number of timestamps up to the current time.
For each of the non-overlapping $\nctime \ll \ntime$ timestamps, we can obtain $\tensorC\shapeN^{\nctime\times\tensorShape}$ as the partial tensor of $\tensor$.
Our goal is to summarize continuously growing $\tensor$ and detect group anomalies by processing it incrementally as subtensor $\tensorC$.
\TSK{
    \begin{figure*}[t]
    \centering
    \includegraphics[width=\linewidth]{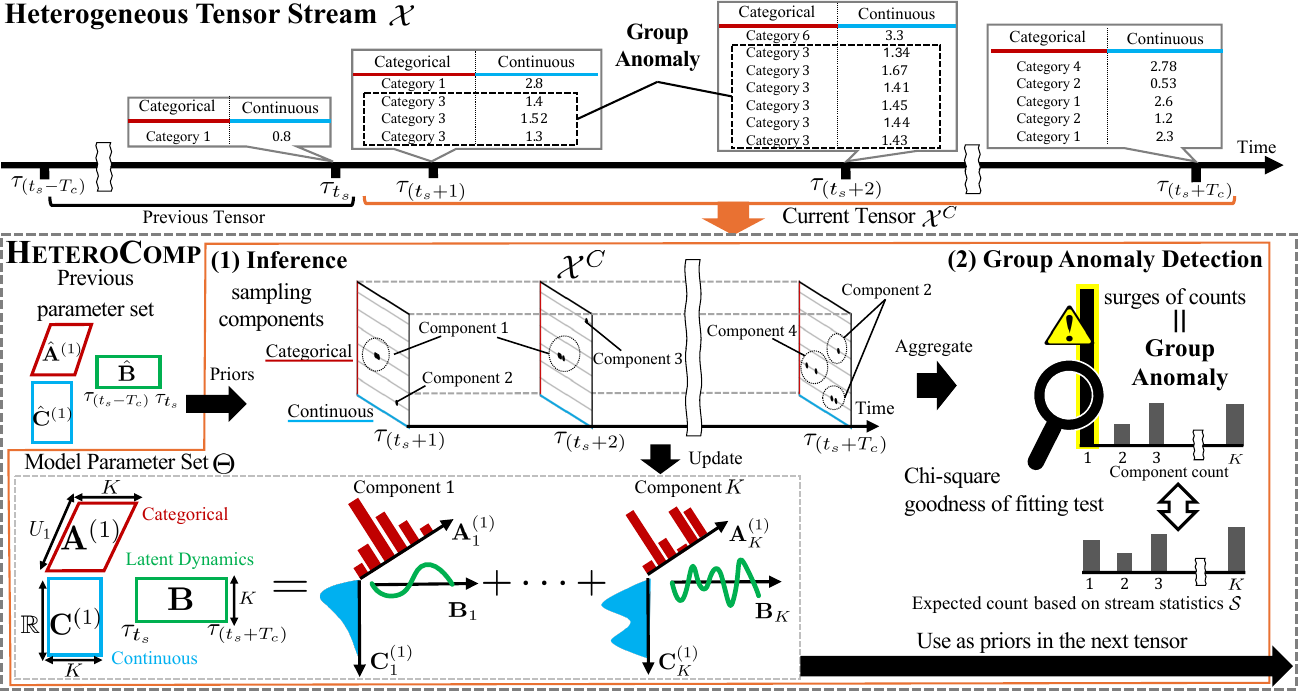}
    \caption{Illustration of \method: Given a current tensor $\tensorC$ consisting one \dmode and one \cmode, (1) it assigns \topics to each \record in $\tensorC$ and update model parameter $\Matt^{(1)}, \Mcatt^{(1)}, \Mtime$, (2) it quickly and accurately detects group anomalies based on \topics counts.}
    \Description{Figure showing how \method decompose the \streamNames.}
    \label{fig:overview}
    \vspace{-10pt} 

\end{figure*}
}
\subsection{Gaussian Process}
To model the latent dynamics of \topic and the distribution of \cmodes, we use Gaussian processes \cite{GPML}, which  are distributions over functions, fully specified by a kernel covariance  function $\KernelFunc$,  such that any finite set of function values is jointly Gaussian.
Specifically,  for a finite set of inputs $X=\{x_1,\dots,x_n\}$, the function values are distributed as
    \begin{equation}
      f\sim GP( 0, \KernelFunc) \overset{\mathrm{def}}{\Longleftrightarrow}
      (f(x_1), \dots, f(x_n)) \sim \text{Normal}(\mathbf 0, \mathbf
      K),
    \end{equation}
    where $\mathbf
    K_{i,j} = \KernelFunc(x_i, x_j)$ is a covariance matrix.
\subsection{Model}
We now present our model in detail.
We assume that there are $\ntopic$ major trends behind the event collections and refer to such trends as ``\topic''.
As shown in Fig. \ref{fig:overview},  the $\ltopic$-th \topic is characterized by the following probability distributions.
\bit[leftmargin=*]
    \item $\Matt_\ltopic^{(\ldmode)}\shapeR^{\nunits_{\ldmode}}$ : Multinomial distribution over the $\ldmode$-th \dmode, for the \topic $\ltopic$.
    \item $\Mtime_{\ltopic}\shapeR^{\nctime}$ : Latent dynamics for the \topic $\ltopic$.
        \begin{equation}
            \Mtime_\ltopic\sim GP( 0, \Ktime).
        \end{equation}
    \item $\Mcatt_\ltopic^{(\lcmode)}\shapeR^{\ngrid_\lcmode}$ : Distribution for $\lcmode$-th \cmode, the \topic $\ltopic$.
    To model an unknown distribution in $\R$,  we use a logistic Gaussian process (LGP) prior:
    \begin{align}
        \Mcatt_\ltopic^{(\lcmode)} &\sim GP( 0, \Kcatt),\\
            p_{LGP} (\event^{(\lcmode)}) &= \frac{
            \exp(\Mcatt_\ltopic^{(\lcmode)}(\event^{(\lcmode)}))
            }{\int_\mathbb{R} \exp(\Mcatt_\ltopic^{(\lcmode)}(e'))de'}.
        \label{eq_lgp}
    \end{align}
    We discretize $\mathbb R$ into $\ngrid_\lcmode$ non-overlapping grids  $\{\grid_{\lgrid}\}_{\lgrid=1}^{\ngrid_\lcmode}$ to improve the efficiency of inference.
    In other words, we use the following probability distribution instead of Eq.  (\ref{eq_lgp}) :
    \begin{align}
        p_{LGP} (\event^{(\lcmode)}\in \grid_{\lgrid}) &\approx \frac{|\grid_{\lgrid}|\exp(\gridy^{(\lcmode)}_{\ltopic, \lgrid})}{\sum_{g'=1}^{\ngrid_{\lcmode}}|\grid_{\lgrid'}|\exp(\gridy^{(\lcmode)}_{\ltopic,\lgrid'})},
        \label{eq_lgp_approx}
    \end{align}
    where $|\grid_{\lgrid}|$ is the width of the $\lgrid$-th grid, and $\gridy^{(\lcmode)}_{\ltopic,\lgrid}$ is a predicted value at the center point of $\grid_{\lgrid}$.
\eit
\begin{definition}[Model parameter set: $\Parameters$]
\label{def:model_parameters}
Let $\Parameters = \bigl\{\{\Matt^{(\ldmode)}\}_{\ldmode=1}^{\ndmode},\allowbreak \Mtime,\allowbreak \{\Mcatt^{(\lcmode)}\}_{\lcmode=1}^\ncmode\bigr \}$ be a parameter set of \method in  $\tensorC$.
\end{definition}
We also incorporate temporal dependencies into this model so that each parameter captures the context of its predecessors in the data stream.
Specifically, we assume that the means of $\Matt, \Mcatt$ are the same as $\Patt, \Pcatt$, which are the parameters estimated at the previous tensor, unless the newly arrived tensor $\tensorC$ are confirmed.
With this assumption, we can use $\text{Dirichlet}(\dpersistency\Patt_\ltopic)$ for \dmodes, and $\text{Normal}(\Pcatt_\ltopic^{(\lcmode)},  \varc I)$ for \cmodes.
$\dpersistency$ is a hyperparameter representing the temporal persistence of the $\ldmode$-th \dmode.
\footnote {We set $\dpersistency = \frac1\ntopic$ as default.}
Consequently, as shown in the graphical model in Fig. \ref{fig:graphical},  the generative process of $\tensorC$ can be described as follows:
\vspace{5pt}

{
    \hspace{-1em}
    \fbox{
        \begin{minipage}{0.95\columnwidth}
        \begin{itemize}[leftmargin=*]
            \item For each \topic $\ltopic = 1, ..., \ntopic$:
                \begin{itemize}
                    \item For each \dmode $\ldmode = 1, ..., \ndmode$:
                    \begin{itemize}
                        \item $\Matt_\ltopic^{(\ldmode)} \sim \text{Dirichlet}(\dpersistency\Patt_\ltopic^{(\ldmode)})$
                    \end{itemize}
                    \item For each \cmode $\lcmode = 1, ..., \ncmode$:
                    \begin{itemize}
                        \item $\Mcatt_\ltopic^{(\lcmode)} \sim \text{Normal}\Big(\Pcatt_\ltopic^{(\lcmode)},  \varc I\Big)$
        \end{itemize}
                    \end{itemize}
                \end{itemize}
                 \begin{itemize}[leftmargin=*]
                     \item For each time $\ltime = \startTime+1, ..., \startTime+\nctime$:
                \begin{itemize}
                    \item For each \record $\levent = 1, ..., \nevent$:
                    \begin{itemize}
                        \item $\assignment \sim \text{Categorical}\left(\text{softmax}(\Mtime(\stamp_\ltime))\right)${\color{black}\ //  \Topic. }
                        \item For each \dmode $\ldmode = 1, ..., \ndmode$:
                        \begin{itemize}
                            \item $\event^{(\ldmode)}_{\ltime, \levent} \sim \text{Categorical}\left(\Matt_{\assignment}^{(\ldmode)}\right)$
                        \end{itemize}
                        \item For each \cmode $\lcmode = 1, ..., \ncmode$:
        
                        \begin{itemize}
                            \item $\event_{t, \levent}^{(\lcmode)} \sim  p_{LGP}(\event^{(\lcmode)}_{t, \levent} |\Mcatt_{\assignment}^{(\lcmode)})$
                            {\color{black}\ // Eq.(\ref{eq_lgp_approx})}
            \end{itemize}
                        \end{itemize}
                    \end{itemize}
        
            \end{itemize}
    \end{minipage}
    }
}

\\
where $\nevent$ is the total number of \records at time $\stamp_\ltime$, and  $\assignment$ is the \topic assignment.
We note that the benefits of this model are three-fold.
First, our model can summarize arbitrary-order heterogeneous sparse tensors into $\ntopic$ \topics.
Second, our model employs Gaussian process for $\Mtime$ and $\Mcatt$,
enabling unified handling of \cmodes and capturing complex dynamics of \topics.
Lastly, to capture temporal dependencies, it employs the parameters of the previous tensor rather than storing tensors.
\begin{figure}[t]
    \centering
    \includegraphics[width=\linewidth]{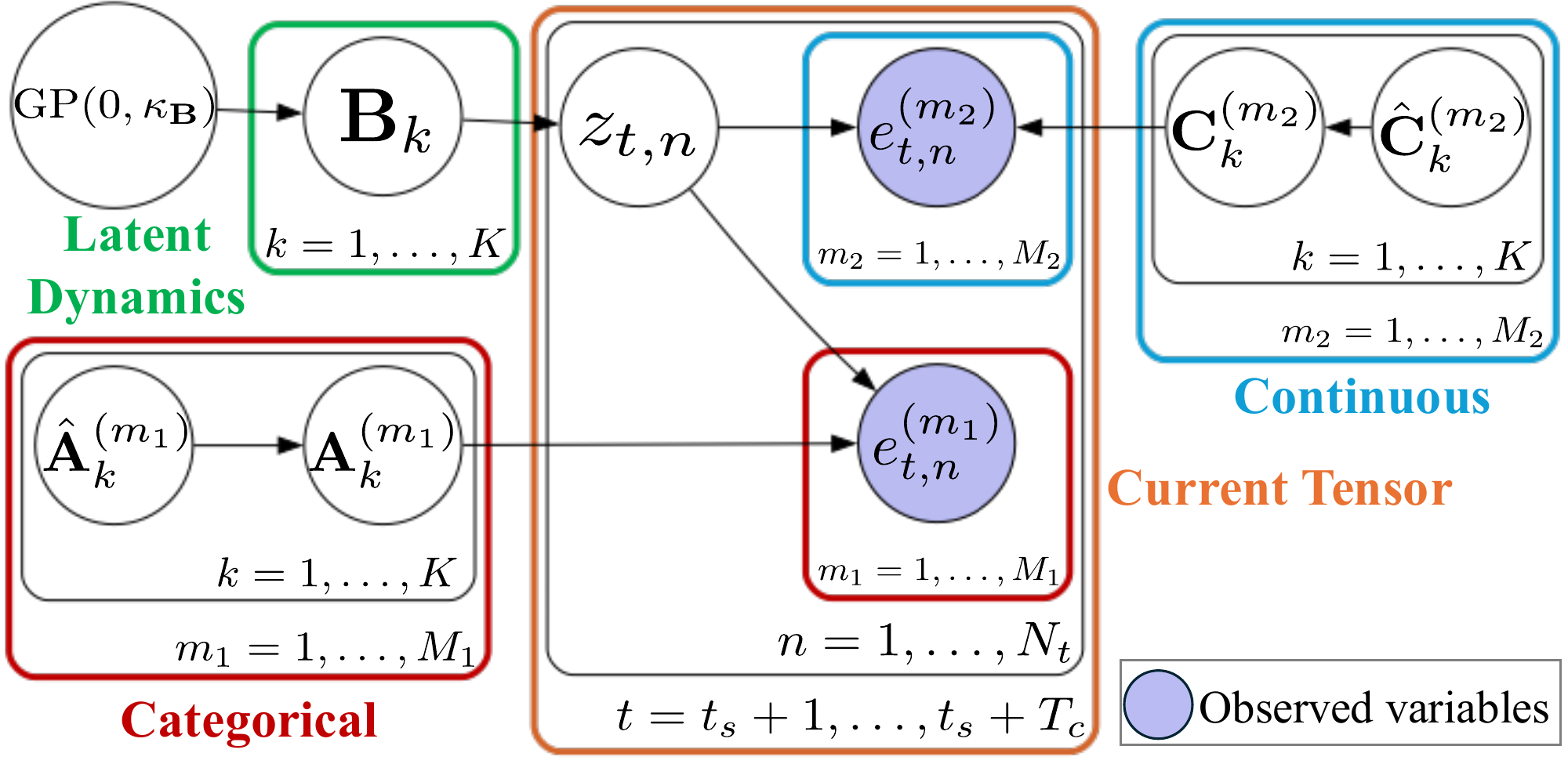}
    \caption{Graphical model of \method.}
    \Description{Figure showing the graphical model of \method.}
    \label{fig:graphical}
\end{figure}

\section{Algorithm}
    Before turning to the main topic, we introduce statistics for anomaly detection.
\begin{definition}[Stream Statistics: $\Statistics $]
\label{def:stream_statistic}
Let $\Statistics =\bigl\{\sinterval,  \Stopic,\\ \{\Stopicmode\}_{\ldmode=1}^{\ndmode}, \{\Stopicgrid\}_{\lcmode=1}^\ncmode\bigr\}$ be a statistics set of the entire stream $\tensor$.
$\sinterval$  is the total normal time, and $\Stopic\in \N^\ntopic$ is the total \topic count vector of normal \records.
$\Stopicmode \shapeN^{\ntopic \times \nunits_\ldmode}$ and $\Stopicgrid \shapeN^{\ntopic \times \ngrid_\lcmode}$ denote the total \topic{}–unit and \topic{}–grid count matrices of normal records, respectively.
\end{definition}
With the above definitions, the formal problem is as follows:
\begin{problem}
\label{prob1}
\textbf {Given} the current tensor  $\tensorC$ as a partial tensor of $\tensor$,
\bit
    \item \textbf{Estimate} the model parameter set $\Parameters$.
    \item \textbf {Maintain} the stream statistics $\Statistics$ of $\tensor$,
    \item \textbf{Report} the anomaly score for $\tensorC$,
\eit
incrementally and quickly, at any point in time.
\end{problem}
In this section, we present practical algorithms for solving Problem \ref{prob1}.
Algorithm \ref{alg:main} shows the overall procedure.
Specifically, we first infer the \topic assignments and model parameter set $\Parameters$ from $\tensorC$, and then calculate the anomaly score of $\tensorC$.

\subsection{Inference}
\TSK{
    
\begin{algorithm}[t]
    \renewcommand{\algorithmicrequire}{\textbf{Input:}}
    \renewcommand{\algorithmicensure}{\textbf{Output:}}
    \small
    \caption{\method $(\tensorC, \Parameters)$}
    \label{alg:main}
    \begin{algorithmic} [1]
        \REQUIRE{\begin{tabular}[t]{@{}l@{\,}l} 
              1. & Current tensor: $\tensorC\shapeN^{\nctime \times \tensorShape}$ \\
              2. & Previous model parameter set: $\hat\Parameters$\\
              3. & Previous stream statistics: $\Statistics$
            \end{tabular}   }
        \ENSURE {\begin{tabular}[t]{@{}l@{\,}l} 
              1. &Updated model parameter set: $\Parameters$\\
              2. &Updated stream statistics: $\Statistics$\\
              3. &Anomaly score: $\anomscore$
            \end{tabular}
            }
        \STATE {\color{blue} /* Inference */} 
        \FOR {{\bf each} iteration}
            \FOR {{\bf each} \record in $\tensorC$}
                \STATE Draw \topic $\assignment$  {\color{black}// Eq.~(\ref{eq:topic_sample})}
            \ENDFOR
            \FOR {$\ltopic = 1, ..., \ntopic$}
                \FOR {$\ltime = \startTime + 1, ..., \startTime+\nctime$}
                    \STATE Draw polya gamma $\polyagamma_{\ltopic,\ltime}$ {\color{black}// Eq.~(\ref{eq:pg_sample})}
                    \STATE Caluculate $\pgmean, \pgvar$
                    {\color{black}// Eq.(\ref{eq:pg_update_sigma}),(\ref{eq:pg_update_mu})}
                \ENDFOR
                \STATE Estimate $\Mtime_{\ltopic}$
                {\color{black}// Eq.\eqref{eq:forward_predmean}  - \eqref{eq:backward_smoothcov}}
            \ENDFOR
        \ENDFOR
        \STATE Estimate $\{\Matt^{(\ldmode)}\}_{\ldmode=1}^{\ndmode}, \{\Mcatt^{(\lcmode)}\}_{\lcmode=1}^\ncmode$ {\color{black}// Eq. \eqref{eq:update_Mactt}, \eqref{eq:update_Matt} }
        \STATE $\Parameters\leftarrow (\{\Matt^{(\ldmode)}\}_{\ldmode=1}^{\ndmode}, \Mtime, \{\Mcatt^{(\lcmode)}\}_{\lcmode=1}^\ncmode )$
        \STATE {\color{blue} /* Group Anomaly Detection */} 
        \STATE Calculate $\anomscore$ and p-value.{\color{black}// Eq. \eqref{eq:anomscore}}
        \IF {p-value $< 0.05$ \color{black}} 
            \STATE \textbf{Report} anomaly
        \ELSE  
            \STATE Add $\cinterval, \Ntopic, \{\Ntopicmode\}_{\ldmode=1}^\ndmode, \{\Ntopicgrid\}_{\lcmode=1}^\ncmode$ to $\Statistics$
        \ENDIF\\
        \RETURN  {$\Parameters$, $\Statistics$, $\anomscore$}
    \end{algorithmic}
    \normalsize
\end{algorithm}

}
According to the generative process, we efficiently estimate parameters by employing collapsed Gibbs sampling \cite{CGS}.
Specifically, we repeatedly estimate $\Mtime$ and sample \topics of each \record, and then estimate $\Matt$ and $\Mcatt$ after the \topics have  converged.

\myparaitemize{Sampling \topics}  We sample \topics for each \record in $\tensorC$ according to the following probability:
{
    \begin{align}
     &p(\assignment=\ltopic| \cdot) \propto \text{softmax} \Big(\Mtime(\stamp_\levent)\Big)_\ltopic \times\notag\\
    &\prod_{\ldmode=1}^\ndmode\frac{
    \Ntopicmode[\ltopic, \event_{\ltime,\levent}^{(\ldmode)}]{}'
    + \dpersistency\Patt_{\ltopic, \event_{\ltime, \levent}^{(\ldmode)}}
    }{\Ntopic[\ltopic]'+ \dpersistency}
    \prod_{\lcmode=1}^\ncmode p_{LGP} (\event_{\ltime,\levent}^{(\lcmode)} |\Mcatt_\ltopic^{(\lcmode)}),
    \label{eq:topic_sample}
    \end{align}
}
where $\Ntopic[\ltopic]$ is the number of \records assigned to \topic $\ltopic$ and $\Ntopicmode[\ltopic,\lunit]$ is the total counts \topic $\ltopic$ is assigned to the $\lunit$-th unit.
The prime (e.g., $\Ntopic[\ltopic]{}'$) indicates the count yielded by excluding the \record $\event_{\ltime,\levent}$.

\myparaitemize{Estimate $\Mtime$}
After sampling \topics, we estimate the posterior distribution of $\Mtime$ based on the sampled \topics.
However, we cannot directly use Gibbs sampling because a softmax function is non-conjugate. To address this, we used the Polya-Gamma data augmentation trick \cite{Polson2013-sf} to conjugate the softmax function.
First, for each \topic and for each time, we sample $\polyagamma_{\ltime,\ltopic}$  according to the following equation:
\begin{align}
    \label{eq:pg_sample}
    \polyagamma_{\ltime,\ltopic} &\sim PolyaGamma(\Ntimetopic, \obsmean_{\ltime,\ltopic}-\xi_{\ltime,\ltopic}), \\
        \xi_{\ltime\ltopic} &= \log \sum_{j\ne, \ltopic}e^{\obsmean_{j, \ltime}},
\end{align}
where $\obsmean_{\ltime,\ltopic}$ is prior mean of $\Mtime_\ltopic(\stamp_\ltime)$ and $\Ntimetopic$ is the number of \records assigned to the \topic $\ltopic$ at time $\stamp_\ltime$.
Then, we compute posterior mean $\pgmean$ and variance $\pgvar$.
\begin{align}
    \label{eq:pg_update_sigma}
    \pgvar &= \Big((\obsvar_{\ltime,\ltopic})^{-1} + \polyagamma_{\ltime,\ltopic}\Big)^{-1},
    \\
    \label{eq:pg_update_mu}
    \pgmean &= \pgvar\Big((\obsvar_{\ltime,\ltopic})^{-1} \obsmean_{\ltime,\ltopic}+ \Ntimetopic - \frac{\Ntime}2 +\polyagamma_{\ltime,\ltopic}\xi_{\ltime,\ltopic}\Big),
\end{align}
where $\obsvar_{\ltime,\ltopic}$ is the prior variance of $\Mtime_\ltopic(\stamp_\ltime)$.
\par
Next, we perform Gaussian process regression for $\{\pgmean\}_{\ltime=\startTime+1}^{\startTime+\nctime}$ to estimate the posterior of $\Mtime_\ltopic$.
However, naive computation takes $O(\nctime^3)$ time, so we approximate $\Mtime_\ltopic$ as a linear time-invariant stochastic differential equation (LTI-SDE):
{
\small
\setlength{\abovedisplayskip}{5pt}
 \setlength{\belowdisplayskip}{5pt}
\begin{align}
    \mathbf{x}_\ltopic(t) &= \Big(\Mtime_\ltopic(t), \frac{d \Mtime_\ltopic(t)}{dt}, ..., \frac{d^\Norder \Mtime_\ltopic(t)}{dt^\Norder}\Big), \\
    \Mtime_\ltopic(t) &= \Matobservation \mathbf{x}_\ltopic(t) + \obsnoise_\ltime \quad ( \obsnoise_\ltime \sim \text{Normal}(0, \noisevar)),\\
    \frac{d \mathbf{x}_{\ltopic}(t)}{dt} &= \mathbf{F}\mathbf{x}_\ltopic(t) + \mathbf{L}\mathbf{w}(t),
    \label{eq:sde}
\end{align}
\normalsize
}
where $\Norder$ is the order of the derivative, $\mathbf{w}(t)\shapeR^{s}$ is a multivariate white noise process with a spectral density matrix $\mathbf{Q}_{noise}\shapeR^{s\times s}$, and $\mathbf{F}\in\R^{(\Norder+1)\times (\Norder+1)}, L\in\R^{(\Norder+1)\times s},\Matobservation\shapeR^{1\times(\Norder+1)}$ is a feedback, a noise effect, and an observation matrix.
Many covariance functions
can be expressed as (\ref{eq:sde}) equivalently or approximately \cite{Sarkka2019-zm}.
For discrete values, this translates into
\begin{align}
    \mathbf{x}_{\ltopic}(\stamp_{\ltime}) &= \lineartrans_{\ltime-1} \mathbf{x}_{\ltopic}(\stamp_{\ltime-1}) + \mathbf{q}_{\ltime-1},\quad \mathbf{q}_{\ltime-1}\sim \text{Normal}( \mathbf 0, \mathbf{Q}_{\ltime-1}),
    \label{eq:ssm}
    \\\nonumber
    \lineartrans_{\ltime-1} &= e^{\mathbf{F}(\stamp_{\ltime}-\stamp_{\ltime-1})},
    \qquad
    \mathbf{Q}_{\ltime-1}
    = \mathbf{P}_{\infty} - \lineartrans_{\ltime-1}\mathbf{P}_{\infty} \lineartrans_{\ltime-1}^T.
\end{align}
The initial state is distributed according to
$\mathbf{x}_{\ltopic}(\stamp_{\startTime})\sim \text{Normal}\allowbreak(\smoothmean_{\ltopic, \startTime}, \allowbreak\smoothcov_{\ltopic, \startTime})$, where $\smoothmean_{\ltopic, \startTime}, \smoothcov_{\ltopic, \startTime}$ are the  posterior parameters in the previous tensor, computed according to Eq. (\ref{eq:backward_smoothmean}),  (\ref{eq:backward_smoothcov}) \footnote{We employ $\filtmean_{\ltopic, \startTime}=\mathbf{0}$ and $\filtcov_{\ltopic, \startTime} = \mathbf{P}_0$ at the first tensor.}.
Stationary covariance  $\mathbf{P}_\infty$ can be found by solving the Lyapunov equation:
\begin{align}
    \mathbf{F}\mathbf{P}_{\infty} + \mathbf{P}_{\infty} \mathbf{F}^T + \mathbf{L}\mathbf{Q}_{noise} \mathbf{L}^T = \mathbf 0.
\label{eq:riccati}
\end{align}
By approximating the Gaussian process regression as in Eq.  (\ref{eq:ssm}), the regression problem can be solved with $O(\nctime\Norder^3)$ time complexity and $O(\nctime\Norder^2)$ memory complexity using a Kalman filter and a Rauch-Tung-Striebel (RTS) smoother.
Specifically, the forward filtering update formula is obtained as follows:
\begin{align}
    \label{eq:forward_predmean}
    \predmean_{\ltopic,\ltime}&=\lineartrans_{\ltime-1}\filtmean_{\ltopic, \ltime-1},  \\
    \label{eq:forward_predcov}
    \predcov_{\ltopic, \ltime}&=\lineartrans_{\ltime-1}\filtcov_{\ltopic,\ltime-1} \lineartrans_{\ltime-1}^T + \mathbf{Q}_{\ltime-1},\\
    \label{eq:forward_gain}
    \Gain_{\ltopic,\ltime} &= \predcov_{\ltopic, \ltime} \Matobservation^T \pgvar^{-1},\\
    \label{eq:forward_filtmean}
    \filtmean_{\ltopic,\ltime} &= \predmean_{\ltopic,\ltime}+ \Gain_{\ltopic,\ltime}(\pgmean - \Matobservation\predmean_{\ltopic,\ltime}),\\
    \label{eq:forward_filtcov}
    \filtcov_{\ltopic,\ltime} &= (\mathbf{I}- \Gain_{\ltopic,\ltime})\predcov_{\ltopic,\ltime},
\end{align}
where $\predmean_{\ltopic, \ltime}, \predcov_{\ltopic, \ltime}$ is a predicted mean and covariance of $\mathbf{x}_\ltopic(\stamp_\ltime)$.
Furthermore, the backward smoothing update formula is also obtained as follows:
\begin{align}
    \mathbf{J}_{\ltopic,\ltime}&=   \filtcov_{\ltopic,\ltime} \lineartrans_{\ltime-1}^T(\predcov_{\ltopic,\ltime+1})^{-1},
    \label{eq:backward_gain}\\
    \smoothmean_{\ltopic,\ltime} &= \filtmean_{\ltopic,\ltime} +\mathbf{J}_{\ltopic,\ltime} (\smoothmean_{\ltopic,\ltime+1} - \lineartrans_{\ltime-1}\filtmean_{\ltopic,\ltime}),
    \label{eq:backward_smoothmean}\\
    \smoothcov_{\ltopic,\ltime} &=\filtcov_{\ltopic,\ltime} + \mathbf{J}_{\ltopic,\ltime} (   \smoothcov_{\ltopic,\ltime+1} - \predcov_{\ltopic,\ltime}  )\mathbf{J}_{\ltopic,\ltime}^T
    \label{eq:backward_smoothcov},
\end{align}
where $\smoothmean_{\ltopic, \ltime}, \smoothcov_{\ltopic, \ltime}$ is a smoothed mean and covariance of $\mathbf{x}_\ltopic(\stamp_\ltime)$.
To summarize, we use $\text{Normal}(\predmean_{\ltopic,\ltime}, \predcov_{\ltopic,\ltime})$ as the distribution for $\textbf{x}_\ltopic(\stamp_\ltime)$ in the first epoch, and $\text{Normal}(\smoothmean_{\ltopic,\ltime}, \smoothcov_{\ltopic,\ltime})$ in subsequent epochs.

\myparaitemize{Estimate $\Matt$}
After Gibbs sampling has burned in, we compute the posterior of $\Matt$.
Because the Dirichlet distribution is conjugate to the Categorical distribution, we can compute $\Matt$ analytically, as follows:
\begin{equation}
    \label{eq:update_Matt}
    \Matt^{(\ldmode)}_{\ltopic, \lunit}=\frac{\Ntopicmode[\ltopic,\lunit] + \dpersistency\Patt^{(\ldmode)}_{\ltopic,\lunit}}{\Ntopic[\ltopic] + \dpersistency},
\end{equation}
where $\Ntopicmode[\ltopic,\lunit]$ is the total count of \topic $\ltopic$ that is assigned to the $\lunit$-th unit in the $\ldmode$-th \dmode.

\myparaitemize{Estimate $\Mcatt$}
After Gibbs sampling has burned in, we compute the posterior of $\Mcatt$.
For convenience, we denote $|\grid_\lgrid|$ as $w_\lgrid$.
First, we approximate $\Mcatt$ as LTI-SDE, similar to $\Mtime$.
Next, to estimate the unknown density, we use MAP estimation, which aims to find $\mathbf{\gridy}_{\ltopic}$ that maximizes the following log-likelihood:
{
    \begin{align}
        \nonumber
        L(\mathbf{\gridy}_{\ltopic})=& \sum_{\lgrid=1}^{\ngrid_\lcmode} \Ntopicgrid[\ltopic,\lgrid] \Big(\log w_\lgrid + \gridy_{\ltopic, \lgrid}\Big)
        \label{eq:log_llh_lgp}
        - \Ntopic[\ltopic] \log \Big(\sum_{\lgrid=1}^{\ngrid_\lcmode}w_\lgrid \exp(\gridy_{\ltopic, \lgrid})  \Big)\notag\\
        &- \frac{1}{2}\sum_{\lgrid=1}^{\ngrid_\lcmode} \Big(\log |2\pi \inovationVar| + \inovationMean(\inovationVar)^{-1}\inovationMean\Big),
    \end{align}
}
where $\Ntopicgrid[\ltopic,\lgrid]$ is the total count of \topic $\ltopic$ that is assigned to the $\lgrid$-th grid in the $\lcmode$-th \cmode, and
$\inovationMean = \gridy_{\ltopic, \lgrid}- \Matobservation \predmean_{\ltopic, \lgrid}$, $\inovationVar =\Matobservation \predcov_{\ltopic, \lgrid}\Matobservation^T + \noisevar$ is an innovation mean and its variance.
The $\lgrid$-th element of the gradient is as follows:
\begin{equation}
    \label{eq:update_Mactt}
    \Bigl(\nabla L(\mathbf{\gridy}_{\ltopic})\Big)_\lgrid= \Ntopicgrid[\ltopic,\lgrid] - \Ntopic[\ltopic]\frac{w_\lgrid \exp(\gridy_{\ltopic, \lgrid})}{\sum_{\lgrid'=1}^{\ngrid_\lcmode}w_{\lgrid'} \exp(\gridy_{\ltopic, \lgrid'})} -(\inovationVar)^{-1}\inovationMean.
\end{equation}
We can estimate $\mathbf{\gridy}_{\ltopic}$ using the L-BFGS method\cite{LBFGS} because Eq.  (\ref{eq:log_llh_lgp}) has a unique maximum.

\TSK{
    \begin{figure*}
    \centering
    \includegraphics[width=\linewidth]{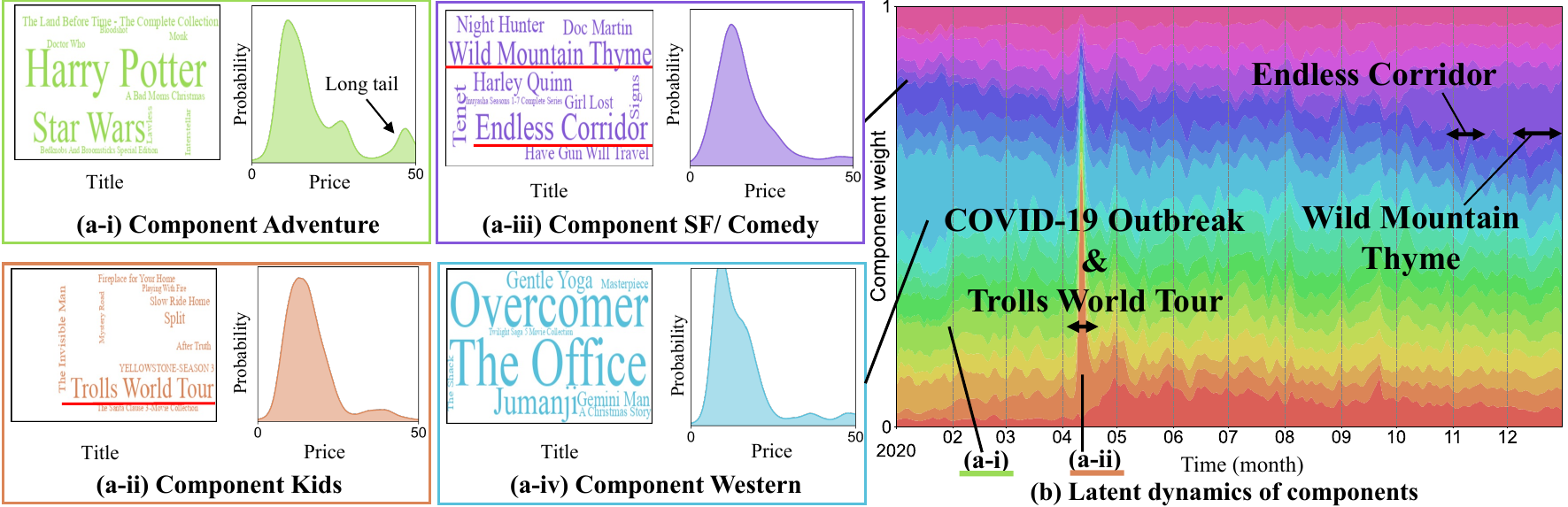}
    \caption{Market analysis of \method in the \amazon dataset. (a) The characteristics of four \topics (Adventure, Kids, SF/Comedy, Western) in \dmode (Title) and \cmode (price in US dollars). 
    (b) Component weight exhibit significant changes in relation to the film’s release.}
    \Description{Modeling power of \method.}
    \label{fig:amazon_movie}
\vspace{-5pt} 
\end{figure*}
}

\subsection{Group Anomaly Detection}
We exploit inferred \topics to calculate an anomaly score of $\tensorC$.
Our goal is to detect group anomalies, so we treat a sudden surge in a \topic's occurrence or an abrupt increase in an \attribute's count within a \topic as a group anomaly.
Therefore, we execute the chi-squared goodness-of-fit test.

\begin{description}
    \item [$\mathbf{H_0}$:]
    The null hypothesis assumes that the average of $\Ntopic$, $\Ntopicmode$,  $\Ntopicgrid$ in the current tensor is the same as the average of them in all previous normal times.
    \begin{align}
        \label{eq:expected_count_topic}
        \mathbb{E}[\Ntopic] &= (\Ntopic+ \Stopic)\frac{\cinterval}{\sinterval+ \cinterval},\\
        \mathbb{E}[\Ntopicmode] &= (\Ntopicmode+ \Stopicmode)\frac{\cinterval}{\sinterval+ \cinterval},\\
        \label{eq:expected_count_topicgrid}
        \mathbb{E}[\Ntopicgrid] &= (\Ntopicgrid+ \Stopicgrid)\frac{\cinterval}{\sinterval+ \cinterval},
    \end{align}
    where $\cinterval= \stamp_{(\startTime+\nctime)}-\stamp_{\startTime}$ is the interval of the current tensor.
    \item [$\mathbf{H_1}$:]The alternative hypothesis assumes that at least one of $\Ntopic$, $\Ntopicmode$, $\Ntopicgrid$ contains anomalies, and observed counts show a statistically significant difference from the expected counts, namely, Eq.  (\ref{eq:expected_count_topic}) - (\ref{eq:expected_count_topicgrid}).

\end{description}

Using the chi-squared statistic, $\chi_m^2(\mathbf{x}, M) = \sum_{m=1}^M (\mathbf{x}_m - \mathbb{E}[\mathbf{x}_m])^2 \allowbreak/ \mathbb{E}[\mathbf{x}_m]$, we define the anomaly score as follows:
\begin{align}
    \label{eq:anomscore}
    \anomscore &= \chi_\ltopic^2(\Ntopic, \ntopic)\notag+ \sum_{\ldmode=1}^\ndmode\sum_{\ltopic=1}^\ntopic\chi_\lunit^2(\Ntopicmode[\ltopic, \cdot], \nunits_\ldmode)\\
    &+\sum_{\lcmode=1}^\ncmode\sum_{\ltopic=1}^\ntopic\chi_\lgrid^2(\Ntopicgrid[\ltopic,\cdot] , \ngrid_\lcmode).
\end{align}
\begin{lemma}[Proof in Appendix \ref{sec:proof_anomaly}]
    \label{lemma_anomaly}
    $\anomscore$ follows a chi-squared distribution with $\ntopic (\sum_{\ldmode}^\ndmode\nunits_\ldmode + \sum_{\lcmode}^\ncmode\ngrid_\lcmode-\ndmode-\ncmode+1)-1$ degrees of freedom.
\end{lemma}
This lemma indicates that we can compute the p-value $P(X > \anomscore)$. If the p-value is less than $0.05$, we reject the null hypothesis, and thus, $\tensorC$ is judged as an anomaly.
Otherwise, the null hypothesis cannot be rejected, so we judge $\tensorC$ as normal, and add $\cinterval, \Ntopic, \{\Ntopicmode\}_{\ldmode=1}^\ndmode, \{\Ntopicgrid\}_{\lcmode=1}^\ncmode$ to $\Statistics$.

\myparaitemize{Time complexity of \method}
Lemma \ref{lemma_complexity} indicates that our proposed algorithm
requires only constant computational time for the entire data stream length $\ntime$, thus \method is practical for semi-infinite data streams in terms of execution speed.
\begin{lemma}[Proof in Appendix \ref{sec:proof_complexity}]
\label{lemma_complexity}
The time complexity \method for the current tensor $\tensorC$  is
$O(\nepoch N_{event}\ntopic (\ndmode + \ncmode)+\nepoch\ntopic \nctime\Norder^3+\sum_{\ldmode=1}^\ndmode \nunits_\ldmode \ntopic+\nlfbgs \ntopic\sum_{\lcmode=1}^\ncmode\ngrid_\lcmode\Norder^3)$, where $\nepoch$ is the epoch count of \topic sampling, $N_{event} = \sum_{\ltime=1}^\nctime \nevent$ is the number of \records, and $\nlfbgs$ is the L-BFGS iteration count.
\end{lemma}

\section{Experiments}
    \label{section:experiments}
    In this section,  we evaluate the performance of \method.
The experimental settings are detailed in Appendix \ref{sec:experiment_appendix}.
We answer the following questions through the experiments.
\begin{itemize}[leftmargin=2.5em]
    \item[(Q1)]\textit{Effectiveness:}
        How successfully does it discover interpretable summarization of real datasets?
    \item[(Q2)] \textit{Accuracy:}
        How accurately does it detect group anomalies from real datasets?
    \item[(Q3)] \textit{Scalability:}
        How does it scale in terms of computational time?
\end{itemize}
\TSK{
    \begin{figure}
    \centering
    \includegraphics[width=\columnwidth]{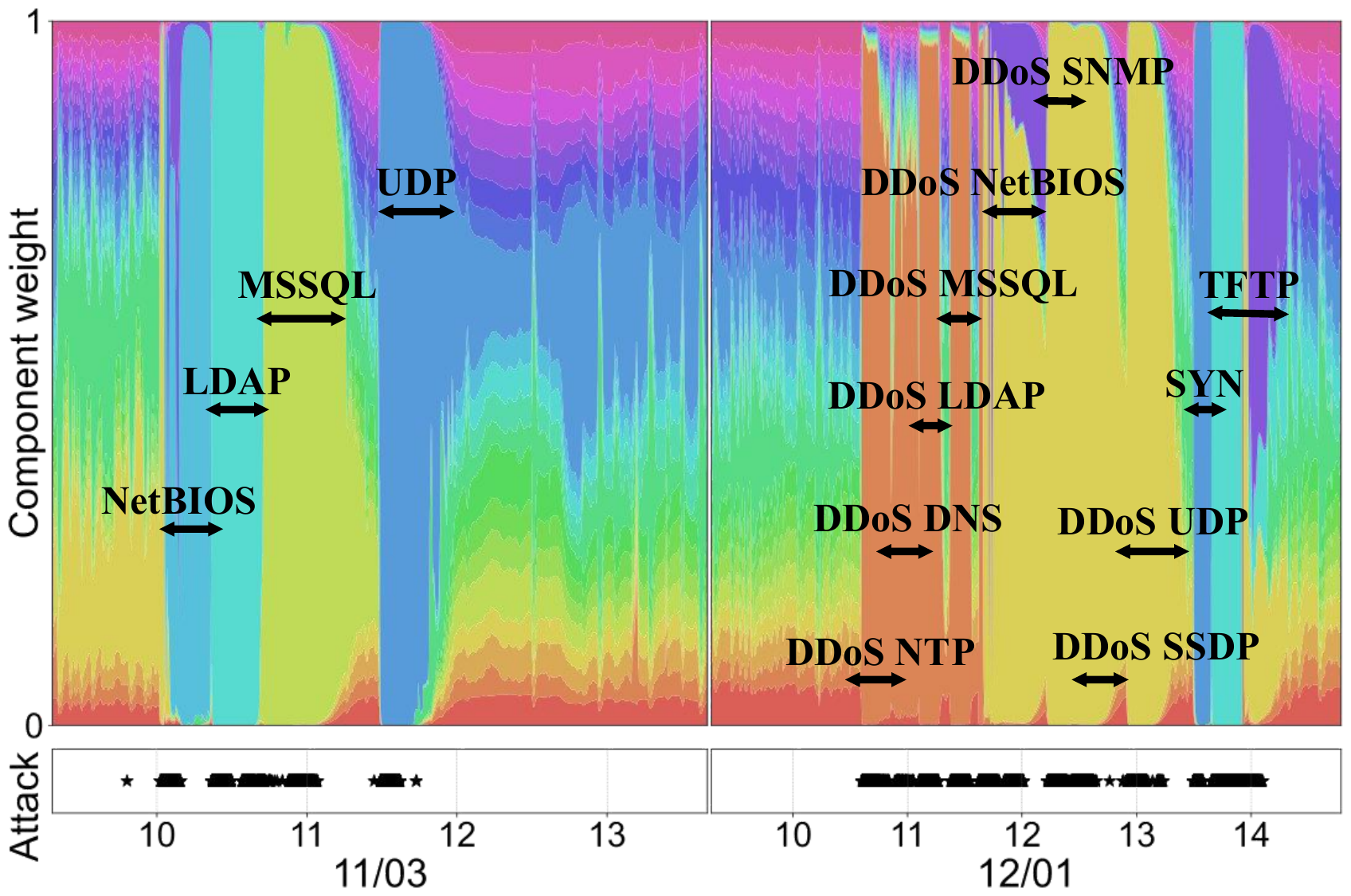}
    \caption{Dynamics of $\Mtime$ (above) and attacked time (below) in \pddos dataset.}
    \Description{In the top figure, the x-axis represents time, and the y-axis represents the assignment probability for each topic. In the bottom figure, the x-axis represents time, with a star marking each point of attack.}
    \label{fig:ddos_B}
\end{figure}
    \begin{table*}[ht!]
\caption{\textbf{Anomaly detection results. Best results are in bold, and second-best results are underlined (higher is better). The rightmost column shows the average value for each metric.}}
\label{tbl:accuracy}
\resizebox{1.0\linewidth}{!}{   
\begin{tabular}{l|cc|cc|cc|cc|cc|cc}
\toprule
 & \multicolumn{2}{c|}{\ciseven \cite{CI17}} & \multicolumn{2}{c|}{\cieight \cite{CI18}} & \multicolumn{2}{c|}{\edge \cite{edge}} & \multicolumn{2}{c|}{\ddos \cite{DDOS2019}} & \multicolumn{2}{c|}{\cupid \cite{CUPID}} & \multicolumn{2}{c}{\textbf{Average}} \\
 & \roc & \pr & \roc & \pr & \roc & \pr & \roc & \pr & \roc & \pr & \roc & \pr \\
\midrule
OneClassSVM \cite{OCSVM} & $0.587$ & $0.082$ & $0.594$ & $0.146$ & $0.662$ & $0.601$ & $0.900$ & $0.913$ & $0.467$ & $0.016$ & $0.642$ & $0.352$ \\
iForestASD \cite{IForestASD} & $0.844\pm0.001$ & $0.540\pm0.003$ & $\underline{0.781\pm0.001}$ & $\underline{0.428\pm0.005}$ & $0.700\pm0.009$ & $0.680\pm0.008$ & $0.881\pm0.001$ & $0.912\pm0.001$ & $0.957\pm0.004$ & $0.608\pm0.003$ & $0.833$ & $0.634$ \\
RRCF \cite{RRCF} & $0.877\pm0.002$ & $0.679\pm0.004$ & $0.763\pm0.009$ & $0.337\pm0.010$ & $0.927\pm0.002$ & $0.919\pm0.003$ & $0.896\pm0.002$ & $0.922\pm0.001$ & $0.974\pm0.004$ & $0.705\pm0.013$ & $0.888$ & $0.712$ \\
ARCUS \cite{ARCUS} & $0.500\pm0.002$ & $0.028\pm0.002$ & $0.503\pm0.004$ & $0.153\pm0.010$ & $0.501\pm0.001$ & $0.586\pm0.178$ & $0.500\pm0.002$ & $0.280\pm0.014$ & $0.497\pm0.008$ & $0.013\pm0.000$ & $0.500$ & $0.212$ \\
MStream \cite{MStream} & $0.905\pm0.000$ & $0.736\pm0.000$ & $0.779\pm0.000$ & $0.363\pm0.000$ & $0.928\pm0.000$ & $0.927\pm0.000$ & $0.899\pm0.000$ & $0.925\pm0.000$ & $0.991\pm0.000$ & $0.734\pm0.000$ & $0.900$ & $0.737$ \\
MemStream \cite{MemStream} & $0.893\pm0.000$ & $0.713\pm0.000$ & $0.781\pm0.000$ & $0.366\pm0.000$ & $\underline{0.935\pm0.000}$ & $\mathbf{0.935\pm0.000}$ & $0.950\pm0.002$ & $0.956\pm0.001$ & $0.977\pm0.000$ & $0.678\pm0.000$ & $0.907$ & $0.730$ \\
Anograph \cite{Anograph} & $\underline{0.921\pm0.000}$ & $\underline{0.741\pm0.000}$ & $0.776\pm0.001$ & $0.419\pm0.005$ & $0.928\pm0.000$ & $0.920\pm0.000$ & $\mathbf{0.974\pm0.000}$ & $\mathbf{0.970\pm0.000}$ & $\underline{0.994\pm0.000}$ & $0.814\pm0.001$ & $\underline{0.915}$ & $\underline{0.773}$ \\
CubeScope \cite{CubeScope} & $0.921\pm0.001$ & $0.545\pm0.003$ & $0.490\pm0.002$ & $0.123\pm0.001$ & $0.294\pm0.004$ & $0.421\pm0.004$ & $0.684\pm0.013$ & $0.715\pm0.010$ & $0.986\pm0.000$ & $\underline{0.872\pm0.004}$ & $0.675$ & $0.535$ \\
CyberCScope \cite{CyberCScope} & $0.625\pm0.037$ & $0.302\pm0.096$ & $0.659\pm0.090$ & $0.202\pm0.047$ & $0.771\pm0.054$ & $0.633\pm0.056$ & $0.502\pm0.176$ & $0.574\pm0.127$ & $0.940\pm0.034$ & $0.785\pm0.116$ & $0.699$ & $0.499$ \\
\midrule
\textbf{\method ({ours})} & $\mathbf{0.990\pm0.006}$ & $\mathbf{0.931\pm0.037}$ & $\mathbf{0.788\pm0.005}$ & $\mathbf{0.644\pm0.008}$ & $\mathbf{0.935\pm0.003}$ & $\underline{0.931\pm0.003}$ & $\underline{0.963\pm0.003}$ & $\underline{0.970\pm0.002}$ & $\mathbf{0.999\pm0.000}$ & $\mathbf{0.959\pm0.001}$ & $\mathbf{0.935}$ & $\mathbf{0.887}$ \\
\bottomrule
\end{tabular}
}
\end{table*}

}
\myparaitemize{Datasets}
We used five real network traffic/intrusion datasets, namely \pciseven \cite{CI17}, \pcieight \cite{CI18}, \pedge \cite{edge},  \pddos \cite{DDOS2019},\pcupid \cite{CUPID}, and one user-review dataset, namely \pamazon \cite{AmazonReview}.

\myparaitemize{Baselines}
We undertook comparisons with the following competitors for streaming anomaly detection: OneClassSVM \cite{OCSVM}, iForestASD \cite{IForestASD},  RRCF \cite{RRCF}, ARCUS \cite{ARCUS}, Mstream \cite{MStream},  MemStream \cite{MemStream},  Anograph \cite{Anograph}, CubeScope \cite{CubeScope}, and CyberCScope \cite{CyberCScope}.
\subsection{Q1: Effectiveness}
\label{subsec_experiment_effective}
We first demonstrate how effectively \method discovers interpretable summarization on real datasets.

\myparaitemize{User Review Analysis}
Fig. \ref{fig:amazon_movie} shows our mining result for \pamazon dataset.
First, Fig. \ref{fig:amazon_movie}(a) shows the characteristics of four \topics, where we manually
named them "Adventure", "Kids", "SF/Comedy", "Western", in title (i.e., $\Matt$) and price (i.e., $\Mcatt$).
From the word clouds, Adventure is associated with blockbuster franchises such as Harry Potter and Star Wars. It exhibits a right-skewed price distribution with a high-price tail, whereas Kids (e.g., Trolls World Tour) concentrates in the low–to–mid price range. SF/Comedy (e.g., Endless Corridor, Wild Mountain Thyme, Tenet) and Western (e.g., The Office, Overcomer, Jumanji) films also appear at relatively low prices, with Western having the lowest median price. These observations indicate that each component captures coherent semantics in $\Matt$ while exhibiting distinct regimes in the continuous attribute $\Mcatt$.
Next, Fig. \ref{fig:amazon_movie}(b) shows the latent dynamics (i.e., $\Mtime$), namely, which \topics the users are interested in 2020.
In April, a spike of \topic Kids is observed, which is due to the closure of movie theaters caused by the COVID-19 pandemic and Universal’s release of the animation film "Trolls World Tour" via streaming starting on April 10, 2020.
Similarly, the release of "Endless Corridor" and "Wile Mountain Thyme" increase the users' attention to \topic SF/Comedy.

\myparaitemize{Cybersecurity systems}
As discussed in Section \ref{section:introduction}, Fig. \ref{fig:preview} showed that \method can effectively estimate \topics with various distributional characteristics and their temporal changes (dynamics).
Additionally, Fig. \ref{fig:ddos_B} contrasts the temporal variation of the estimated $\Mtime$ with the actual time of the cyberattack in the \pddos dataset.
During the attacks, a specific \topics increases sharply.
Please also see the results in \pciseven and \pcieight  datasets in Appendix \ref{appendix_effective}.

These results show that \method can capture the interpretable \topics and their temporal dynamics consistent with external events, such as movie releases or cyber-attacks.

\subsection{Q2: Accuracy}
We next evaluate the accuracy of \method in terms of group anomaly detection.
We defined a current tensor as anomalous if it contains more than one hundred anomalous \records.
For point-anomaly detection methods, the anomaly score of a current tensor was defined as the sum of the anomaly scores of the \records contained within it.
Table \ref{tbl:accuracy} shows \roc and \pr for each method, where a higher value indicates better detection accuracy.
All results are averaged over three runs with  random seeds.
Our proposed method \method achieves the highest average detection accuracy across the datasets, which demonstrates that \method is effective for group anomaly detection.
Point-anomaly detection methods (One Class SVM, iForestASD, RRCF, ARCUS) exhibit low performance in detecting group anomalies.
MStream and MemStream achieve high detection accuracy, but since they cannot exploit the continuity of timestamps, they suffer from many false positives, resulting in low \pr, especially on \pcieight dataset.
Although Anograph obtains better performance with \pddos, it underperforms \method across other datasets because it is a graph-based method and cannot handle multi-aspect data.
CubeScope discretizes \cmodes, and CyberCScope assumes \cmodes follow only a Gamma distribution, which limit their abilities to represent diverse \attribute distributions and degrade detection performance.
\subsection{Q3: Scalability}
Finally, we verify the computation time of \method.
The left part of Fig. \ref{fig_scalability} shows the average wall clock time of an experiment performed on three datasets, \pciseven, \pedge, \pddos.
Although there are slight variations due to differences in the number of \records, thanks to the incremental update, \method can maintain stable computational performance independent of the overall stream length.
The right part of Fig. \ref{fig_scalability} shows the computational time of \method when varying the number of \records in $\tensorC$.
Since \method achieves fast model estimation for $O(N_{event})$ time (as discussed in Lemma \ref{lemma_complexity}), its computation time is linear with respect to the number of \records (i.e., slope = 1 in log-log scale).
\TSK{

\begin{figure}[t]
    \centering
    \begin{subfigure}[b]{0.47\columnwidth}
        \centering
        \includegraphics[width=\linewidth]{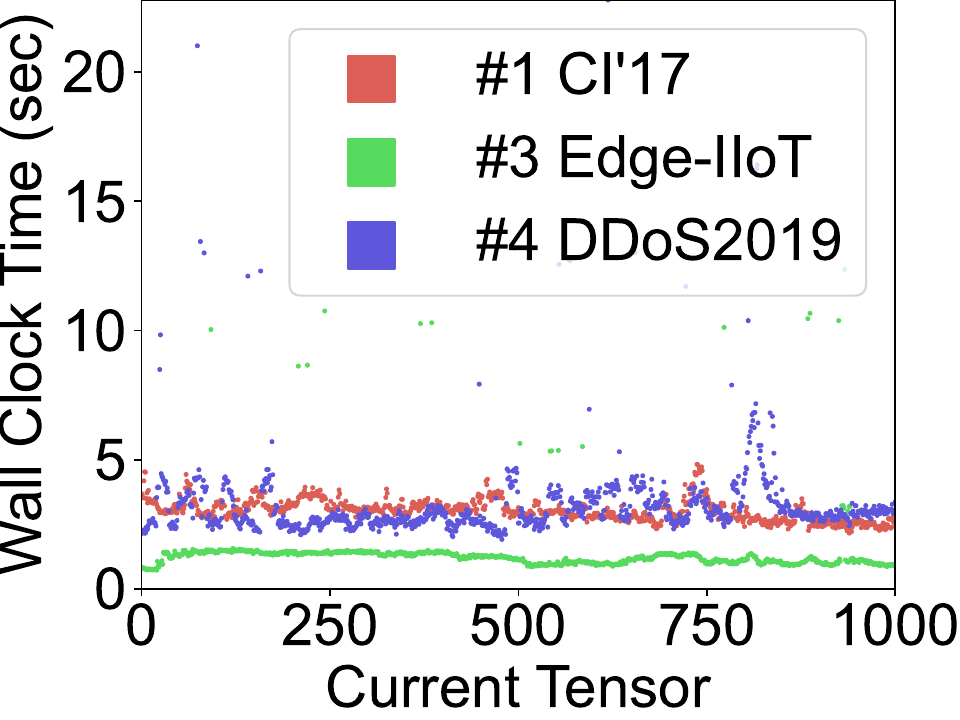}
        \label{subfig_elapsed}
    \end{subfigure}
    \hfill
    \begin{subfigure}[b]{0.47\columnwidth}
        \centering
        \includegraphics[width=\linewidth]{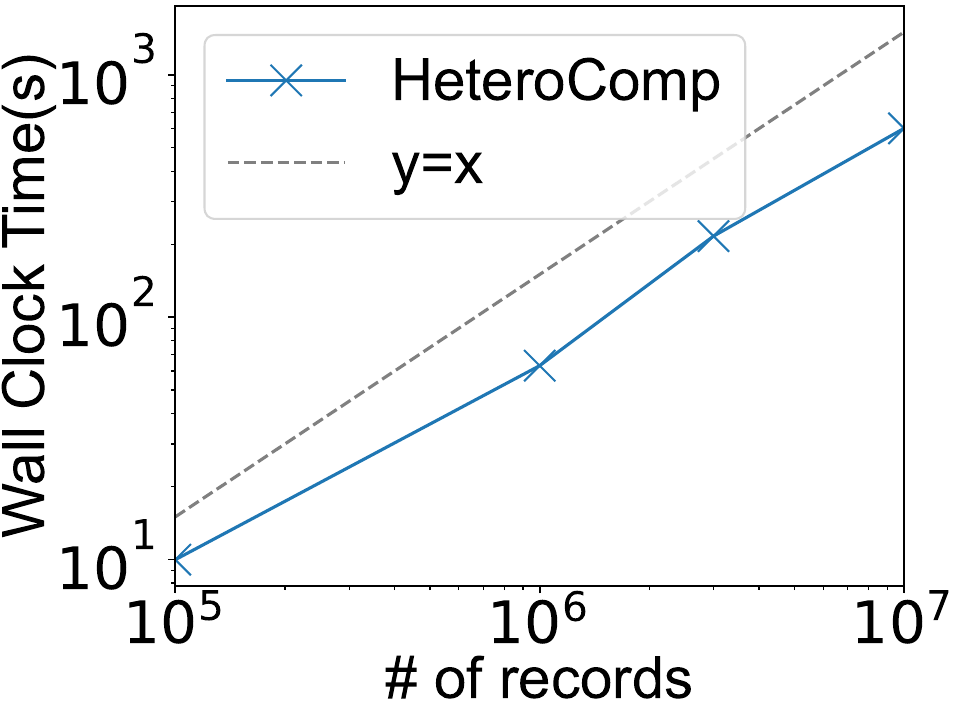}
        \label{subfig_linear_scalability}
    \end{subfigure}
    \vspace{-10pt}
    \caption{Complexity analysis of \method: (Left) Average wall clock time vs. current tensor. (Right)  Wall clock time vs. \# of records in $\tensorC$. The algorithm scales linearly (i.e., slope = 1 in log-log scale).}
    \Description{In the left figure, the x-axis represents the current tensor index and the y-axis represents its processing time. The processing time is parallel to the x-axis. In the right figure, the x-axis represents the number of events in the tensor and the y-axis represents its processing time. The processing time increases linearly with the number of events.}
    \label{fig_scalability}
\end{figure}

}

\section{Conclusion}
    \label{section:conclusion}
    In this paper, we propose \method, which summarizes \streamNames and detects anomalies in real-time.
\method can simultaneously \topics and their temporal dynamics without restricting to any specific parameterized form of \cmodes, and immediately detect anomalous behavior based on them.
Our approach exhibits all of the following desirable properties that we listed in the introduction. (a)  \textit{Effective}: It discovers latent \topics and their latent dynamics. (b) \textit{Accurate}: Our experiments demonstrated that \method detects group anomalies accurately.
(c) \textit{Scalable}: The computational time does not depend on the data stream length.

\begin{acks}
We would like to thank the anonymous referees for their valuable and helpful comments.
This work was supported by
JSPS KAKENHI Grant-in-Aid for Scientific Research Number JP24KJ1618,
JST CREST JPMJCR23M3,
JST START JPMJST2553,
JST CREST JPMJCR20C6,
JST K Program JPMJKP25Y6,
JST COI-NEXT JPMJPF2009,
JST COI-NEXT JPMJPF2115,
the Future Social Value Co-Creation Project - Osaka University.
\end{acks}
\bibliographystyle{lib/ACM-Reference-Format}
\balance
\bibliography{bib/references}
\appendix
\section*{Appendix}
\label{section:appendix}
\section{Proposed model}
\subsection{Symbols}
The main symbols we use in this paper are defined in Table \ref{table:define}.
\TSK{
    \begin{table}[h]
\centering
\caption{Symbol and its definition.}
\label{table:define}
\resizebox{1.0\linewidth}{!}{
    \begin{tabular}{l|l}
        \toprule
        Symbol & Definition \\

        \midrule
        $\tensor$ & Whole event tensor stream,\\
            & i.e., $\tensor\shapeN^{\ntime\times\tensorShape}$.\\
        $\tensorC$
            & Current tensor,\\
            & i.e.,$\tensorC\shapeN^{\nctime\times\tensorShape}$.\\ 
        $\ndmode $
            & Number of \dmodes in tensor.\\
        $\ncmode $
            & Number of \cmodes in tensor.\\
        $\nunits_1, \dots, \nunits_{\ndmode}$
            & Number of unique values in \dmode .\\
        ${\nctime}$
            & Number of unique timestamps in current tensor.\\
    
        $\stamp_{(\startTime+1)}, \dots, \stamp_{(\startTime+\nctime)}$
            & Timestamps in $\tensorC$.\\
        $\cinterval$ & Time interval of $\tensorC$, i.e., $\cinterval=\stamp_{(\startTime+\nctime)}-\stamp_{\startTime}$.\\
            
        \midrule
        ${\ntopic}$
            & Number of \topics . \\
        ${\Matt^{(\ldmode)}\shapeR^{{\ntopic}\times{\nunits_{\ldmode}}}}$
            & Parameters of $\ldmode$-th \dmode.\\
        ${\Mcatt^{(\lcmode)}\shapeR^{\ntopic\times\ngrid_{\lcmode}}}$
            & Parameters of $\lcmode$-th \cmode.\\
        ${\Mtime}\shapeR^{\ntopic\times\nctime}$
            & Latent dynamics of \topics.\\
        ${\Patt^{(\ldmode)},\Ptime, \Pcatt^{(\lcmode)}}$
            & ${\Matt^{(\ldmode)}, \Mtime, \Mcatt^{(\lcmode)}}$ estimated at the previous tensor. \\
        $\ngrid_1, \dots, \ngrid_{\ncmode}$
            & Number of grids in \cmode.\\
        $\grid_1, \dots, \grid_{\ngrid_{\lcmode}} $  & Grids of cmode.\\
 $\Ktime, \  \Kcatt$& Kernel covariance function of $\Mtime, \Mcatt$.\\

        \midrule
        $\nevent\in \N$& Number of \records at time $\stamp_\ltime$. \\
        $\Ntopic\in \N^\ntopic$ & \Topic count vector of \records in $\tensorC$.\\
        $\Ntimetopic\in \N$& Number of \records assigned to \topic $\ltopic$ at time $\stamp_\ltime$.\\
        $\Ntopicmode \shapeN^{\ntopic \times \nunits_\ldmode}$ & \Topic{}-unit count matrix for $\ldmode$-th \dmode in $\tensorC$.\\
        $\Ntopicgrid\shapeN^{\ntopic \times \ngrid_\lcmode}$ & \Topic{}-grid count matrix for $\lcmode$-th \cmode in $\tensorC$.\\

        \midrule
        $\Parameters$ & Model Parameters set, defined in Definition \ref{def:model_parameters}.\\
        $\Statistics $ & Stream Statistics, defined in Definition \ref{def:stream_statistic}.\\
        $\anomscore$ & Anomaly score of $\tensorC$.\\
        \bottomrule
    \end{tabular}
}
\vspace{-12pt}
\end{table}

}

\subsection{Proof of Lemma \ref{lemma_anomaly}}
\label{sec:proof_anomaly}
\begin{proof}
$ \chi_\ltopic^2(\Ntopic, \ntopic)$ follows a chi-squared distribution with $\ntopic-1$ degrees of freedom.
After the \topic is assigned, $\chi_\lunit^2(\Ntopicmode[\ltopic, \cdot], \nunits_\ldmode)$ are independently distributed for each \topic, each adhering to a chi-squared distribution with $\nunits_\ldmode-1$ degrees of freedom.
 Similarly, after the \topic is assigned, $\chi_\lgrid^2(\Ntopicgrid[\ltopic, \cdot], \ngrid_\lcmode)$ are independently distributed for each \topic, each adhering to a chi-squared distribution with $\ngrid_\lcmode-1$ degrees of freedom.
The sum of independent chi-squared random variables follows a chi-squared distribution with degrees of freedom equal to the sum of their individual degrees of freedom.
Therefore, $\anomscore$ follows a a chi-squared distribution with $\ntopic (\sum_{\ldmode}^\ndmode\nunits_\ldmode + \sum_{\lcmode}^\ncmode\ngrid_\lcmode-\ndmode-\ncmode + 1)-1$ degrees of freedom .
\end{proof}
\subsection{Proof of Lemma \ref{lemma_complexity}}
\label{sec:proof_complexity}
\begin{proof}
    Derive the time complexity for each element.
    For each iteration, sampling a \topic requires $O(N_{event}\ntopic (\ndmode + \ncmode))$ because we need  $O(\ntopic (\ndmode + \ncmode))$ to compute Equation (\ref{eq:topic_sample}) for each event.
    And estimation of $\Mtime$ requires $O(\ntopic \nctime\Norder^3)$ because the computation of Equation (\ref{eq:pg_update_sigma}, \ref{eq:pg_update_mu}) requires $O(\nctime)$, and Forward-Backward algorithm (\ref{eq:forward_predmean})-(\ref{eq:backward_smoothcov}) requires $O(\nctime\Norder^3)$ for each \topic.
    \par
    Estimation of $\Matt$ requires $O(\nunits_\ldmode \ntopic)$ to compute \eqref{eq:update_Mactt}, so the total cost is $O(\sum_{\ldmode=1}^\ndmode \nunits_\ldmode \ntopic)$.
    \par
    In the estimation of $\Mcatt$, the computation of the gradient \eqref{eq:update_Mactt} requires $O(\ngrid_\lcmode{\Norder}^3)$ for the Forward-Backward algorithm (\ref{eq:forward_predmean})-(\ref{eq:backward_smoothcov}).
    We repeat this computation for each iteration, for each \topic, and for each \cmode, so the total  cost is $O(\nlfbgs\ntopic\sum_{\lcmode=1}^\ncmode\ngrid_\lcmode\Norder^3)$.
    \par
    In the anomaly detection, the computation of Equation (\ref{eq:anomscore}) requires $O(\ntopic + \sum_{\ldmode=1}^\ndmode\nunits_\ldmode\ntopic  + \sum_{\lcmode=1}^\ncmode\ngrid_\lcmode\ntopic)$.
    \par
    Aggregating these, the total computational complexity is $O(\nepoch \allowbreak N_{event}\ntopic (\ndmode + \ncmode)+\nepoch\ntopic \nctime\Norder^3+\sum_{\ldmode=1}^\ndmode \nunits_\ldmode \ntopic+\nlfbgs \ntopic\sum_{\lcmode=1}^\ncmode\ngrid_\lcmode\Norder^3)$.
\end{proof}

    

\begin{figure*}[ht]
    \centering 
    \includegraphics[width=\textwidth]{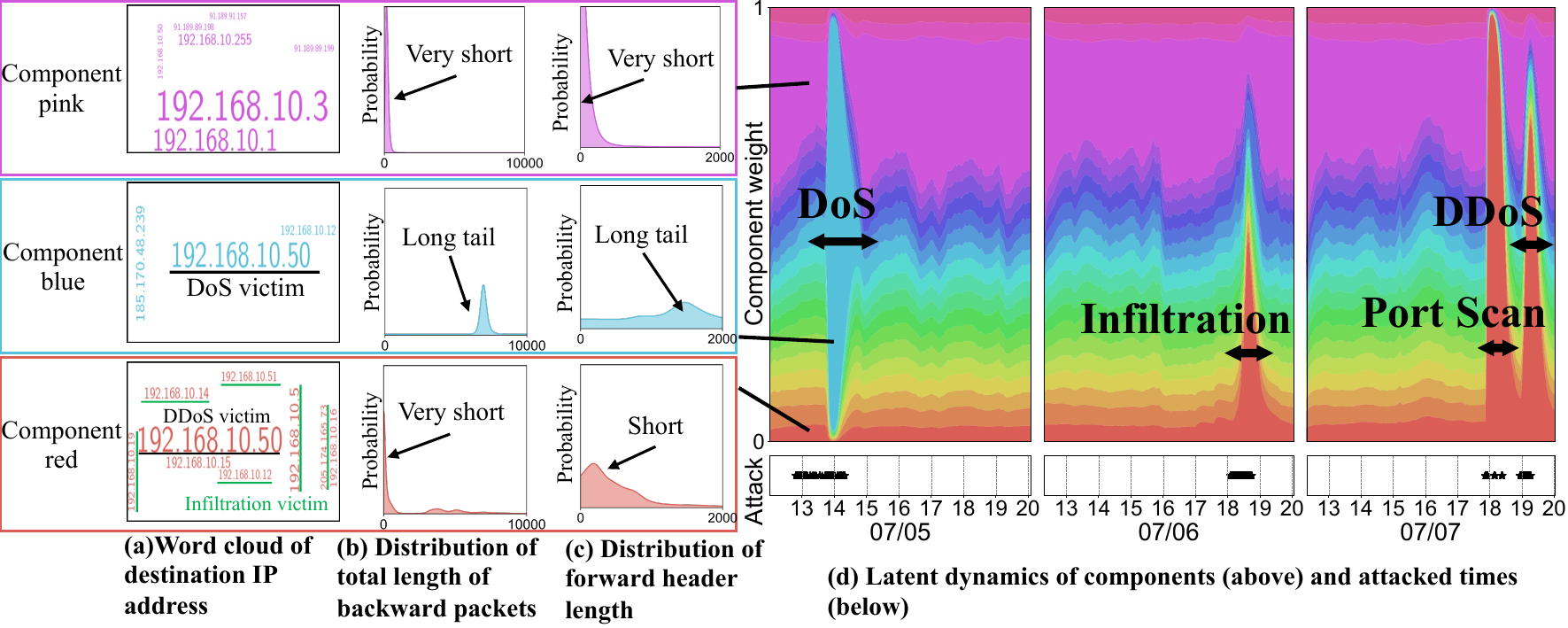}
    \caption{Modeling power of \method over \pciseven dataset.
    Our proposed method can find the hidden \topics which represents different characteristics in both (a) \dmode and (b) (c)\cmodes,  and (d) \topic weight exhibit significant changes when cyber-attacks occurs.
    }
    \Description{}
    \label{fig:anomaly_ciseven}
    \vspace{-10pt} 
\end{figure*}

\section{Experimental Evaluation}
\subsection{Experimental Setup}
\label{sec:experiment_appendix}
In this section, we describe the experimental setup in detail.
We conducted all our experiments on an Intel Xeon Gold 6444Y 3.6GHz quad core CPU with 768GB of memory and running Linux.

\myparaitemize{Hyperparameter}
We use the Matern-3/2 kernel function as a $\Ktime, \Kcatt$.
We set the number of \topics $\ntopic$ to $20$.
We set the number of grids $\ngrid=300$ for all \cmodes, and epoch count $\nepoch=30$.
We set the current tensor size $\nctime$ to $30$ for \pciseven, \pcieight, \pedge, \pddos, \pcupid datasets, and $14$ for \pamazon dataset.

\myparaitemize{Datasets}
Table \ref{table:datasets} summarizes the features employed in the experiments.
\begin{itemize}[leftmargin=*]
    \item \pciseven \cite{CI17}:
        It consists of up to 18 million event logs, in which various types of intrusions occur over time.
        Normal user behavior is executed through scripts.
        The data set contains a wide range of attack types like SSH brute force, heartbleed, botnet, DoS, DDoS, web and infiltration attacks.
        A previous study \cite{ErrorCI} reported errors in this dataset and released improved versions \footnote{\url{https://intrusion-detection.distrinet-research.be/CNS2022/CICIDS2017.html}},
        which we used throughout the experiments.
    \item \pcieight \cite{CI18}: It  includes seven different attack scenarios: Brute-force, Heartbleed, Botnet, DoS, DDoS, Web attacks, and infiltration of the network from inside.
    The attacking infrastructure includes $50$ machines and the victim organization has $5$ departments and includes $420$ machines and $30$ servers.
    The dataset includes the captures network traffic and system logs of each machine, along with 80 features extracted from the captured traffic using CICFlowMeter-V3.
    A previous study \cite{ErrorCI} reported errors in this dataset and released improved versions \footnote{\url{https://intrusion-detection.distrinet-research.be/CNS2022/CSECICIDS2018.html}},  which we used throughout the experiments.
    \item \pedge \cite{edge}:
        It is a new comprehensive cybersecurity dataset for IoT and IIoT applications, designed for intrusion detection systems\footnote{\url{https://ieee-dataport.org/documents/edge-iiotset-new-comprehensive-realistic-cyber-security-dataset-iot-and-iiot-applications}}.
    \item \pddos \cite{DDOS2019}: It is a realistic and comprehensive dataset for evaluating Distributed Denial of Service (DDoS) attack detection systems, as existing datasets have significant shortcomings
    \footnote{\url{https://www.unb.ca/cic/datasets/ddos-2019.html}}.
    \item \pcupid \cite{CUPID}: It emulates a small physical network with several virtualized systems \footnote{\url{https://www.kaggle.com/datasets/dhoogla/cupid-2022}}. Its main objective is to provide both scripted and human-generated traffic produced by professional penetration testers, allowing researchers to investigate the differences between the two.
    \item \pamazon \cite{AmazonReview}: It is a large-scale Amazon Reviews dataset, collected in 2023 by McAuley Lab \footnote{\url{https://amazon-reviews-2023.github.io/}}.
\end{itemize}

\begin{table}[t]
    \vspace{-1.2em}
    \caption{Dataset description. Here, ‘fwd’ and ‘bwd’ denote ‘forward’ and ‘backward,’ respectively.}
    \label{table:datasets}
    \resizebox{1.0\linewidth}{!} {
        \begin{tabular}{l||c|c||c|c}
            \toprule
            Dataset& Categorical& $\ndmode$& Continuous& $\ncmode$\\
            \midrule

            \ciseven \cite{CI17}& (Src/ Dst) IP address& $4$ & Flow duration&  $6$  \\
            & Protocol & & Total length of (fwd /bwd) packet&\\
            &Dst port & & (Fwd /Bwd) header length &\\
            & & &  Flow IAT mean &\\
            \midrule

            \cieight \cite{CI18} &  ''&  $4$ & ''&  $6$ \\
            \midrule
            \edge \cite{edge} &  (Src/ Dst) IP address & $6$ & TCP length& $1$ \\
            & (Src/ Dst) port & & &\\
            & TCP flag& & &\\
            & Protocol & & &\\
            \midrule

            \ddos \cite{DDOS2019}& (Src/ Dst) IP address & $3$ & Flow duration& $6$ \\
            & Protocol& & Total length of fwd packet &\\
            & & &Fwd header length &\\
            & & & ACK flag count &\\
            & & & Flow bytes/s &\\
            & & & Flow packets/s &\\
            \midrule

            \cupid \cite{CUPID}& (Src/ Dst) IP address & $5$ &Flow duration & $7$\\
            & (Src/ Dst) Pport & &Total length of fwd packet &\\
            & Protocol & &Bwd header length &\\
            & & & Flow IAT mean &\\
            & & & Flow bytes/s &\\
            & & & Flow packets/s &\\
            & & & Minimum fwd seg size&\\
            \midrule

            \amazon \cite{AmazonReview} & Title & $23$ & Price  & $1$\\
                    & Category flag(Action, &  &   & \\
                    & Adventure, Anime, &  &   & \\
                    & Arts, Cerebral  &  &   & \\
                    & Christmas, Classical, &  &   & \\
                    & Comedy, Documentary,&  &   & \\
                    & Fantasy, Fitness,  &  &   & \\
                    & Horror, Kids, &  &   & \\
                    & Military, Music, &  &   & \\
                    & Musicals, Mystery,  &  &   & \\
                    & Religion, Romance, &  &   & \\
                    & SF,
Thriller, Westerns)&  &   & \\
            \bottomrule
        \end{tabular}
    }
\end{table}

\myparaitemize{Baselines}
The details of the baselines we used throughout our extensive experiments are summarized as follows:

\begin{itemize}[leftmargin=*]
    \item OneClassSVM(One Class Solid Vector Machine) \cite{OCSVM}: It is a classification principled data stream anomaly detection algorithm.
    We set an upper bound on the fraction to $0.1$, and the learning rate to $0.01$.

    \item iForestASD \cite{IForestASD}: An Anomaly Detection Approach Based on Isolation Forest Algorithm for Streaming Data using Sliding Window.
    Following \cite{MemStream}, we set the window size to $2048$ and the number of estimators to $100$.

    \item RRCF(Robust Random Cut Forest) \cite{RRCF}: Isolation Forest-based method designed for the high-dimensional data anomaly detection problem.
    Following \cite{MemStream}, we set the number of trees to $4$, and  shingle size to $4$, and size of tree to $256$.

    \item ARCUS \cite{ARCUS}: A deep online anomaly detection framework, which uses an adaptive model pool to manage multiple classification models to handle multiple temporal concept drifts.
    We use DAGMM \cite{DAGMM} as instances of ARCUS.
    Following the original paper, we set the batch size to $512$, the learning rate to $0.0001$, the number of layers in AE to $3$, the latent dimensionality of AE to $24$, and the minimum batch size to $32$.

    \item MStream \cite{MStream}: A streaming multi-aspect data anomaly detection framework using locality sensitive hashing. We set the temporal decay factor to $\alpha=0.5$.

    \item MemStream \cite{MemStream}: Streaming approach using a denoising AutoEncoder and a memory module.
    We set the memory size $N = 64$ and the threshold for concept drift $\beta = 0.01$.

    \item Anograph \cite{Anograph}: Graph based streaming anomaly detection method.
    Following the original paper, we set the number of buckets to $32$, edge thresholds to $100$, and time window to $60$.

    \item CubeScope \cite{CubeScope}: An online tensor factorization method based on probabilistic generative models. We set \topic size to $K=48$, as used in \cite{CyberCScope}.

    \item CyberCScope \cite{CyberCScope}: A tensor decomposition method which detects time-varying anomaly patterns while distinguishing between \dmodes and \cmodes. We set \topic size to $K=48$, as used in the original paper.
\end{itemize}

We used open-sourced implementations of ARCUS \cite{ARCUS},  MStream \cite{MStream}, MemStream \cite{MemStream}, CubeScope \cite{CubeScope}, CyberCScope \cite{CyberCScope},  provided by the authors.
For iForestASD \cite{IForestASD} and RRCF \cite{RRCF}, we use the open-source library PySAD \cite{PySAD} implementation.
We also used the open-source implementation of  OneClassSVM \cite{OCSVM} in the river library \cite{montiel2021river}.
For Anograph \cite{Anograph}, we use the open-source implementation \cite{anographRiver} because the original code is implemented in C.

\subsection{Effectiveness}
\label{appendix_effective}
\begin{figure}[t]
     \centering
    \includegraphics[width=\linewidth]{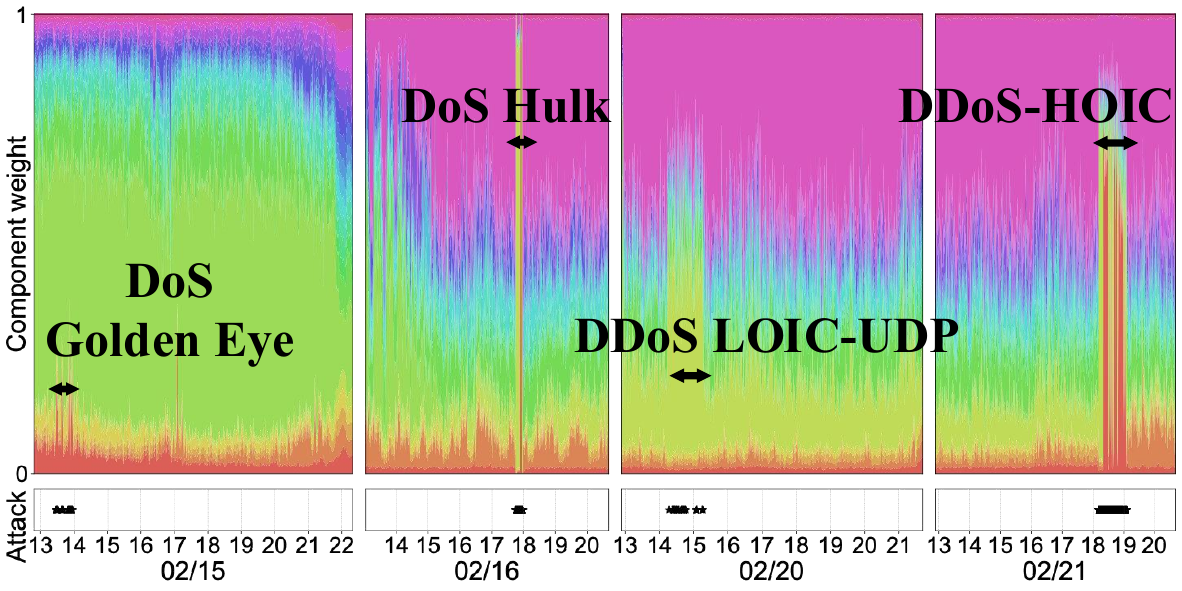}
    \caption{Dynamics of $\Mtime$ in \pcieight. The area of each color represents the \topic assignment probability at each time.}
    \Description{The x-axis represents time, and the y-axis represents the assignment probability for each \component.}
    \label{fig:B_cieight}
    
\end{figure}
We also demonstrate how effectively \method works on datasets different from the ones presented in Section \ref{subsec_experiment_effective}.
Fig. \ref{fig:anomaly_ciseven} shows the analysis of \pciseven dataset.
Fig. \ref{fig:anomaly_ciseven}(a)(b)(c) shows the characteristics of three \topics pink, blue, red.
First, Fig. \ref{fig:anomaly_ciseven}(a) shows the word clouds of destination IP address \attribute (i.e., $\Matt$).
A larger size in the word cloud denote a stronger relationship with the \topic.
\Topic pink contains \records sent to IP address 192.168.10.3 and 192.168.10.3, while \Topic blue consists of \records sent to 192.168.10.50, the victim of the DoS attack.
\Topic red consists of the victims of the Infiltration attack (green underlined) and the victim of DDoS attack (i.e., 192.168.10.50).
Next, Fig. \ref{fig:anomaly_ciseven}(b) and Fig. \ref{fig:anomaly_ciseven}(c) show the probability density of the total length of backward packets and the probability distribution of the forward header length, respectively (i.e., $\Mcatt$).
In Fig. \ref{fig:anomaly_ciseven}(b), \topic pink and red follow exponential-like distributions, whereas \topic blue shows a long-tailed distribution, indicating that \records in pink and red have shorter total backward packet lengths, while those in blue tend to have longer ones.
These results show that \method can flexibly represent various distributions of \cmodes according to the data.
Fig. \ref{fig:anomaly_ciseven} (d) visualizes the latent dynamics of \topics (i.e., $\Mtime$) in the  \pciseven dataset.
During DoS attacks, \topic blue dominates, whereas \topic red increases sharply during Infiltration and Port Scan attacks.
\par
Similarly, Fig. \ref{fig:B_cieight} visualizes the latent dynamics of \topics (i.e., $\Mtime$) and attacked times in the \pcieight dataset.
The proportions of \topics red and green increased sharply during periods of cyber-attacks (e.g., DoS Golden Eye, DoS Hulk, DDoS LOIC-UDP, DDoS-HOIC).
\par
These results show that \method can capture the interpretable \topics in both \dcmodes, and their temporal dynamics consistent with external events, such as cyber-attacks.

\end{document}